\documentclass[11pt]{article}

\usepackage{arxiv}

\usepackage{algorithm,algorithmic}

\usepackage{adjustbox}
\usepackage{graphicx}

\usepackage{array, makecell}
\usepackage{comment}
\usepackage{svg}
\usepackage{subfigure}

\usepackage[utf8]{inputenc} 
\usepackage[T1]{fontenc}    
\usepackage{hyperref}       
\usepackage{url}            
\usepackage{booktabs}       
\usepackage{amsfonts}       
\usepackage{nicefrac}       
\usepackage{microtype}      
\usepackage{xcolor}         
\usepackage{amsthm}
\usepackage{amsmath}

\usepackage{float}
\usepackage{mdframed}
\usepackage{caption}
\usepackage{newfloat}
\usepackage{subfigure}
\newcommand{\mytab}{
\centering
\resizebox{0.35\columnwidth}{!}{
  \begin{tabular}{cccccccccccc}
    \toprule
    Community \# & 0 & 1 & 2& 3 & 4  \\
    $\sf{CC_G}(V_{\ell})$ & -0.14 & 0.07 & 0.2 & -0.15 & -0.09 \\
    $\sf{CC_G}(V_{\ell} \cap S_{0.2})$ & 0.1 & 0.05 & 0.35 & 0.05 & 0.19 \\
    \midrule
    Community \# &  5 & 6 & 7 & 8 & 9 \\
    $\sf{CC_G}(V_{\ell})$ & 0.04 & 0.02 & -0.2 & -0.3 & -0.26\\
    $\sf{CC_G}(V_{\ell} \cap S_{0.2})$ &  0.5 & 0.15 & -0.06 & 0.4 & -0.1\\
    \bottomrule
  \end{tabular}
  }
}

\newcommand{\OO}{\mathcal{O}}

\newcommand{\CC}{\ensuremath{\sf CC}}
\newcommand{\avg}{\ensuremath{\sf average}}

\definecolor{corall}{RGB}{200,100,100}
\definecolor{teall}{RGB}{0,212,212}

\newtheorem{theorem}{Theorem}[section]

\newtheorem{fact}[theorem]{Fact}

\newtheorem{lemma}[theorem]{Lemma}
\newtheorem{definition}[theorem]{Definition}



\definecolor{darkolivegreen}{rgb}{0.33, 0.42, 0.18}
\definecolor{limegreen}{rgb}{0.2, 0.8, 0.2}

\definecolor{darkgreen}{rgb}{0.0, 0.2, 0.13}

\title{A multi-core periphery perspective: Ranking via relative centrality}

\author{ Chandra Sekhar Mukherjee \thanks{Research supported by NSF CAREER award 2141536.}\\
	Thomas Lord Department of Computer Science\\
	University of Southern California\\
\texttt{chandrasekhar.mukherjee07@gmail.com} \\
	\And
Jiapeng Zhang  \thanks{Research supported by NSF CAREER award 2141536.}\\
	Thomas Lord Department of Computer Science\\
	University of Southern California\\
\texttt{jiapengz@usc.edu}
}

\date{\today}

\renewcommand{\undertitle}{}
\renewcommand{\shorttitle}{Ranking in MCPC via relative centrality}

\begin{document}

\maketitle

\begin{abstract}
Community and core-periphery are two widely studied graph structures, with their coexistence observed in real-world graphs (Rombach, Porter, Fowler \& Mucha [SIAM J. App. Math. 2014, SIAM Review 2017]). However, the nature of this coexistence is not well understood and has been pointed out as an open problem (Yanchenko \& Sengupta [Statistics Surveys, 2023]). Especially, the impact of inferring the core-periphery structure of a graph on understanding its community structure is not well utilized. In this direction, we introduce a novel quantification for graphs with ground truth communities, where each community has a densely connected part (the core), and the rest is more sparse (the periphery), with inter-community edges more frequent between the peripheries.

Built on this structure, we propose a new algorithmic concept that we call \emph{relative centrality} to detect the cores. We observe that core-detection algorithms based on popular centrality measures such as PageRank and degree centrality can show some \emph{bias} in their outcome by selecting very few vertices from some cores. We show that relative centrality solves this bias issue and provide theoretical and simulation support, as well as experiments on real-world graphs.

Core detection is known to have important applications with respect to core-periphery structures. In our model, we show a new application: 
relative-centrality-based algorithms can select a subset of the vertices such that it contains 
\emph{sufficient vertices} from all communities, and points in this subset are \emph{better separable} into their respective communities. We apply the methods to 11 biological datasets, with our methods resulting in a more \emph{balanced} selection of vertices from all communities such that clustering algorithms have better performance on this set.
\end{abstract}

\section{Introduction}

Understanding the underlying structures of data is of immense importance in unsupervised learning. Improved 
characterization and exploitation of such structures can lead to better inference of properties in data. 
In this paper, we specifically focus on unsupervised learning on graphs. Here, one of the most popular inference tasks is clustering~\cite{clustering-survey-1,clustering-survey-2}, where one aims to partition a set of vertices into different groups such that vertices in the same group are \emph{similar}. Often, one assumes the vertices have a ground truth partition into some $z$ many true/underlying communities, and the goal is to recover a partition as close as possible to the hidden partition~\cite{modularity-1,SBM-survey-abbe,SBM-survey-directed}. 

In this context, one of the most commonly observed and widely expected structures is that vertices from the same underlying community are more likely to be connected by edges than vertices from different communities. This is often called the \emph{community structure}. This structural observation has inspired important clustering objectives such as modularity~\cite{modularity-1,modularity-2,modularity-3} and correlation clustering~\cite{correlation-clustering} and also popular probabilistic models such as the stochastic block model (SBM)~\cite{SBM-survey-abbe}. Modularity-based clustering algorithms such as Louvain~\cite{louvain} and Leiden~\cite{leiden} are also known to be powerful in real-world settings~\cite{seurat}. 


However, there remain many real-world applications where the state-of-the-art clustering algorithms leave scope for improvement. 
For example, consider single-cell data, which is a type of biological \emph{vector dataset} that has gained popularity in the last 10 years due to applications in immunology~\cite{sc-immunology}, cancer biology~\cite{sc-cancer-biology}, neuroscience~\cite{sc-neuroscience}, and others. Here, one of the most important tasks is to cluster the cells (data) into different cell types (communities). The state-of-the-art clustering algorithm for this domain is widely considered to be Seurat~\cite{seurat}, which is a variant of the application of Louvain to the \emph{shared nearest neighbor} embedding of data, which is an extension of the popular \emph{K-NN embedding}. Though Seurat performs well in some datasets, its performance in hard datasets remains sub-optimal, often with an NMI and purity score of less than 0.7.

In this paper, we aim to explore further structures going beyond community structures that can aid us in developing algorithms to improve downstream applications such as clustering and others. 

We now turn our focus on a different kind of structure called core-periphery (CP) structure, which has been extensively studied in the network analysis literature in the last two decades~(see the surveys \cite{core-periphery-survey,core-periphery-survey-revisited,CP-survey-new} and the references therein). Here, the graph is assumed to contain a dense subgraph (the core), and the other vertices (forming the periphery) are sparsely connected to each other, as well as the core. CP structures are observed widely in social networks, academic citation networks, and others~\cite{cp-sociology-e1,cp-international-e1,cp-international-e2,cp-citation-e1}. A primary task in the CP structure is to \emph{detect the core} in the graph, which has diverse applications in understanding cognitive behavior~\cite{CP-cogonitive}, online amplification~\cite{CP-online}, and others. Here, most of the work is focused on \emph{single-core} setting.
Recently, some papers have observed the existence of multiple cores in graphs~\cite{MCPC-wang11-identify,MCPC-directed,multiple-core-periphery}, and also coexistence of CP and community structure~\cite{MCPC-hybrid-1,MCPC-overlap-is-dense},
with the recent survey \cite{CP-survey-new} noting ``a better understanding of the interplay between community and CP structure'' as one of the open problems. Against this background, our contributions are as follows.


\begin{figure}
    \centering
\includegraphics[scale=0.14]{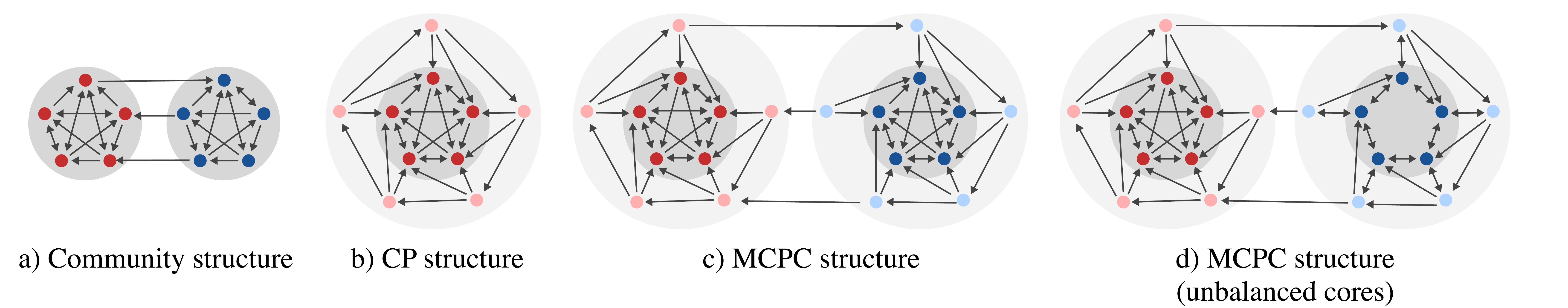}
 \caption{Different structures in $3$-regular directed graphs}
    \label{fig:CP-generic}
\end{figure}

\subsection{Contributions}
\begin{enumerate}
    \item We formalize 
    the coexistence of community and core-periphery structures with
    the \emph{multi-core-periphery with communities} ($\sf{MCPC}$) structure and observe that a large class of real-world data fits this structure.

    \item We design a novel algorithmic concept called \emph{relative centrality} to detect cores in $\sf{MCPC}$ structures and show our algorithms address a previously unexplored unbalancedness issue.

    
    \item We apply our algorithms to a large set of real-world datasets and show that we can select subsets of the data on which clustering algorithms have significantly improved performance.

\end{enumerate}

Now, we elaborate on these points, followed by a comparison with existing works.

\textbf{1. The $\sf{MCPC}$ structure.} Consider a graph $G(V,E)$ with a hidden partition into communities $V_1,\ldots, V_{z}$. We say that the graph satisfies an $\sf{MCPC}$-structure w.r.t the hidden partition if: each set $V_{\ell}$ has a further partition into a \emph{core} $V_{\ell,1}$ and \emph{periphery} $V_{\ell,0}$ such that most edges starting in a core end up in that core itself (the cores are \emph{densely connected}). The peripheries are \emph{loosely connected}, with inter-community edges originating \emph{more frequently} in the peripheries. We give a schematic representation of community structure, CP structure, and MCPC in Figure~\ref{fig:CP-generic} (denoting communities with color and the core with deeper shade). We formalize this intuition in Definition \ref{def: MCPC}. 


\textbf{$\sf{MCPC}$ in the real world:} $\sf{MCPC}$ structure can arise in many interesting real-world scenarios. In a vector dataset with an underlying community structure, if each community has a dense central region and a sparser region around it so that the community overlaps happen in these sparser regions, then we observe that the graph embedding can have $\sf{MCPC}$ structure.  

For example, in single-cell data, each cell is represented by a vector where each dimension corresponds to a particular gene's expression value for this cell. Here, cells belonging to the same cell type will have similar recorded gene expressions. Many cells have a clear cell type, forming a dense region (which will be the core in the $K$-NN embedding). In contrast, the remaining cells have more variance in their expressions, both due to biological and technical factors. For example, CD4-positive and CD8-positive are two important cell types in the immunology literature. However, it has been observed some cells, namely CD4+CD8+ double positive cells, can have similarity of both the cell types~\cite{parel2004cd4+}.
Also, some cells will be more adversely affected by the experimental noise. These cells will not be close to any of the aforementioned dense regions, forming the peripheries in the $K$-NN embedding. 
We give experimental evidence that $\sf{MCPC}$ structure occurs in many single-cell datasets in Section~\ref{sec:single-cell}.


Beyond the applications of this paper, $\sf{MCPC}$ structure may also be relevant to certain social choice theory scenarios. For example, if we consider the social network of US population through the lens of political affiliation, the resultant data should exhibit some $\sf{MCPC}$ structure. This is because there will be people who are firmly committed to one political affiliation, either Democrats or Republicans; these people will form the core of each community. In contrast, more neutral people will align with different affiliations for different topics, forming the peripheries. 
Some recent papers have observed that social networks such as Twitter can have multiple cores~\cite{multiple-core-periphery,multi-cp-structure-2}.

\textbf{Importance of core selection.} 
The first benefit is that if one can detect the cores in a graph satisfying the $\sf{MCPC}$ structure, the induced subgraph becomes more \emph{separable}. This allows downstream applications, such as clustering algorithms, to separate these core vertices better into their respective communities. 
Secondly, we also note that periphery vertices, in many cases, can be meaningful w.r.t to the data. For example, periphery vertices may have similarities to multiple communities (the CD4+CD8+double-positive cells being an example). We believe an improved clustering of the cores into respective communities can also help us understand the implication of periphery vertices better.

\textbf{2. Detecting cores in a balanced manner via relative centrality.} Now our goal is to recover (at least parts of) the cores of each community. In this direction, detecting cores in a single CP structure has a rich literature \cite{CP-survey-new}. Here \emph{centrality measure-based algorithms}
such as degree centrality~\cite{degree-centrality} and PageRank~\cite{pagerank-original} are very popular for ranking vertices according to their coreness, as noted in the comprehensive survey \cite{CP-survey-new} and can be directly applied to graphs with multiple cores. Indeed, we observe that these algorithms perform well by ranking core vertices first in graphs with an $\sf{MCPC}$ structure (we refer to such graphs as $\sf{MCPC}$ graphs).


However, centrality methods can still have a shortcoming when applied to $\sf{MCPC}$ graphs. For example, if we apply the degree centrality to Figure~\ref{fig:CP-generic}(d), the top of the ranking is dominated by red-colored nodes, even though there are two cores (red and blue) of the same size. This implies that degree centrality produces a \emph{biased}/\emph{unbalanced} ranking of the nodes. Note that the downstream algorithm will not know the total number of cores or core vertices. In such a case, if one selects points from the top of the ranking, it may lead to an unbalanced representation of the communities. That is, we may end up selecting very few vertices from a community, and this can make it difficult to interpret/separate that particular community. For example, in single-cell data, one of the goals post-clustering is to observe the variance in the behavior of genes across communities. If we only have a few points of some community in our selected set, such inference may be harder. We observe this bias in traditional centrality measures both in simulation and in a large real-world database.




To mitigate this issue, we propose a new algorithmic framework called \emph{relative centrality} in Section~\ref{sec: relative-centrality}. We show that our algorithm mitigates the biases in the traditional centrality measures both theoretically in an extension of SBM that we describe in Section~\ref{sec:two-block-model} and in various simulations and real-world experiments. We present our concept as a meta-algorithm and observe interesting community structure improvement vs. balancedness tradeoffs between different instantiations. To the best of our knowledge, such a study on balanced core ranking has not been done before.




\textbf{3. Simulations and applications on real-world data.}~For simulation, we apply our algorithms on a natural generalization of the Gaussian mixture model 
in Section~\ref{sec: C-GMM} and demonstrate that \emph{relative centrality} based methods can select a subset of points with better community structure while having \emph{superior balancedness} compared to popular centrality measures such degree centrality and PageRank.


In real-world applications, we use our method to embeddings of a large number of single-cell datasets in Section~\ref{sec:single-cell}. Here, selecting a subset of points using our relative centrality-based algorithms provides a better community structure. This is also reflected in the performance of clustering algorithms, with the NMI and purity of Louvain on the selected points being significantly higher than the original, as recorded in Table~\ref{tab:single-cell-preserve-purity}.
As before, our relative-centrality-based algorithms demonstrate a much higher balancedness compared to traditional centrality measures.



\textbf{Comparison with existing work.} Some papers have made progress in multi-core structures.
However, these works have mainly focused on comparably restrictive settings. The works~\cite{MCPC-hybrid-1} and \cite{multiple-core-periphery} proposed specific equations to capture the interaction between nodes from different cores and designed maximum likelihood-based algorithms for these equations; the paper \cite {MCPC-directed} studied directed graphs with multi-cores, but they do not consider the coexistence of community structure. As such, we applied the algorithms of \cite {MCPC-directed} to our graph simulation model and observed an almost-random outcome and thus placed it in Appendix~\ref{app:comp}, along with a more in-depth comparison with the literature as well as some limitations of our work. Most importantly, to the best of our knowledge, none of the previous papers considered the unbalancedness issue of core selection in the $\sf{MCPC}$ setting.

\section{Core-ranking in \texorpdfstring{$\sf{MCPC}$}~} 
\label{sec:graph-setting}

In this section, we first formalize the notion of multi-core-periphery structure with communities ($\sf{MCPC}$), starting with a definition to capture the ``core-ness'' of a set of vertices. 

\begin{definition}[Core concentration]
\label{def: CC_G}
Given a directed graph $G(V,E)$, for any $V',V'' \subseteq V$, let $E(V',V'')$ denote the number of edges starting in $V'$ and ending in $V''$. Then we define the core concentration (or simply concentration) of $S \subseteq V$  as 
\begin{equation}
\label{eq: CC}
\CC_G(S)= \big( E(\bar{S},S)-E(S,\bar{S}) \big) \big/E(S,V)
\end{equation}
\end{definition}

That is, we expect the core to have two properties. 
First, only a few edges originating in a core should leave the core, which we penalize with the $-E(S,\bar{S})$ term. Secondly, we expect more edges from the peripheries ending at the cores (making each core a more central part of the corresponding community), which we incentivize with the $+E(\bar{S},S)$ term. Note that this value will always be in the range $[-1,E(\bar{S},V)/E(S,V)]$.
Then, we can formally define the $\sf{MCPC}$-structure as follows.





\begin{definition}
\label{def: MCPC}
We say a graph $G(V,E)$ with an underlying hidden partition $\{V_1, \ldots ,V_{z} \}$ has an $(\alpha,\beta)$-$\sf{MCPC}$ structure for $\alpha=\Omega(1)>0$ and $\beta<1$ iff 

i) The hidden partition imposes a community structure on the graph, i.e., $E(V_{\ell},V_{\ell})\ge E(V_{\ell},\bar{V_{\ell}})$.

ii)  Each hidden partition $V_i$ can be further  partitioned into two sets,
namely the core $V_{{\ell},1}$ and the periphery $V_{{\ell},0}$, such that $\sf{CC}_G(V_{{\ell},1})> \sf{CC}_G(V_{{\ell},0})+ \alpha$.

iii) For any two communities $V_{\ell},V_{\ell'}$, 
    $\frac{E(V_{{\ell},1},V_{{\ell}',1})}{E(V_{{\ell},1},V)} \le \beta \frac{E(V_{\ell},V_{\ell'})}{E(V_{{\ell}},V)}$. (There are fewer inter-community edges between cores than the peripheries).
\end{definition}

\paragraph{Validity of our definitions.}
In Section~\ref{sec: C-GMM}, we shall see that $ K$-NN embedding of natural mixture models satisfies our $\sf{MCPC}$ structure definition. In Section~\ref{sec:single-cell}, we observe that the communities in graph embeddings of real-world biological data indeed contain cores with higher concentration values, and these cores have a higher fraction of 
\emph{intra-community edges} compared to the whole data.


Then, our goal is to design core-ranking algorithms with the following properties. 

\begin{definition}[Performance metrics of CR algorithms]
\label{CP:metric} 

A CR algorithm should have a \emph{high value} for the following two metrics.

i) {\it Core-prioritization:} A CR algorithm should rank core vertices above periphery vertices. We quantify this as the AUROC value ~\cite{AUROC-2023} of the w.r.t core/periphery label of each vertex, which is close to $1$ when almost all core vertices are placed above almost all periphery vertices.


ii) {\it Balancedness:} Let $F$ be a ranking of the vertices, and let $F_{c}$ denote the top $c \cdot n$ vertices in the ranking. Then, we quantify the balancedness in the top $c \cdot n$ vertices in the ranking is defined as 
$\mathcal{B}_{c}(F,G):= \left( \min_{\ell} \frac{|F_{c} \cap V_{{\ell},1}|}{|V_{{\ell},1}|} \right) \big / \left( \max_{\ell} \frac{|F_{c} \cap V_{{\ell},1}|}{|V_{{\ell},1}|} \right)$. 
That is, the top points in the ranking should contain a roughly equal proportion of points of each core. The ranking reaches a perfect balancedness if $\mathcal{B}_{c}(F,G)=1$.  The {\bf total balancedness} is defined as 
$\mathcal{B}(F,G):= \frac{1}{n}\sum_{i=1}^{n}\mathcal{B}_{i/n}(F,G)$.
\end{definition}

{\bf Baseline core-ranking algorithms:}
In this direction, we have observed in the introduction with Figure~\ref{fig:CP-generic}(d) that when one core has higher level of connectedness (concentration), degree centrality can give an unbalanced ranking. In this direction, we consider degree centrality, PageRank(with 3 different damping factors 0.5,0.85 and 0.99), Katz Centrality~\cite{katz-and-more}, and a popular core-decomposition-based algorithm~\cite{onion-centrality} (onion decomposition) as initial baseline CR algorithms for $\sf{MCPC}$ structures. Next, we define a random graph model to better study the properties of traditional centrality measures and propose our algorithmic framework.

\subsection{High-level explanations via a random graph block model}
\label{sec:two-block-model}

To better understand the unbalancedness in the centrality measures, we design a simple random graph block model. In a block model, the set of vertices $V$ are divided into some blocks, and then the nature of interaction between two vertices is a function of their block identity. For example, the popular stochastic block model (SBM) is such a model. Here the vertices are divided into some communities, and each pair of vertices are joined with an edge independently with some probability, with intra-community edge probability being higher. This allows one to generate a graph with community structure. Block models are also popular in CP literature, with the performance of many algorithms made under different CP assumptions quantified on these block model-generated graphs~\cite{core-periphery-identify-portar,MCPC-directed,multiple-CP-identify,multi-cp-structure-2}.

We take a similar approach in this paper and define a model that we call $\sf{MCPC}$-block model that generates a directed graph where the out-degree of each vertex is $k(1 \pm o(1))$ for some $k$. The motivation for studying an almost-regular setting is that the real-world examples in this paper, namely the $k$-NN embedding of vector datasets, are regular directed graphs. For simplicity, we define it w.r.t $2$ underlying community, which can be easily extended.

\begin{definition}[$\sf{MCPC}$-block model]
We have $V=\{v_1, \ldots,v_n\}$ that are partitioned into $2$ communities $V_0, V_1$, with each $V_{\ell}$ further partitioned into a core $V_{{\ell},1}$ and a periphery $V_{{\ell},0}$. There is a $4 \times 4$ block-probability matrix $\mathbb{P}$ such that each row sums to $1$. 
Then, for any $v_i \in V_{\ell,c}$ and $v_j \in V_{\ell',c'}$, we add an 
$v_i \rightarrow v_j$ edge 
iid with probability 
$k/|V_{\ell',c'}| \cdot \mathbb{P}[(\ell,c),(\ell',c')]$.
\end{definition}

{\bf Choice of parameters.}
We set $|V_{\ell,c}|=0.25n$ for all the blocks and then set the block parameters to lead to an $(\alpha,\beta)-\sf{MCPC}$ structure. Here, it is important to note that while we demonstrate the phenomenon in theory and in simulation when the size of all the cores and peripheries are the same, our algorithms are applicable even when they are of different sizes. In fact, in our real-world experiments, the size of the different communities in a dataset varies widely. Then, the behavior of degree centrality can be captured as follows. Here we note that the proofs of the theorems in this section can be found in Appendix~\ref{app: proof}.

\begin{theorem}[Behavior of degree centrality]
\label{thm: deg-limitaion}
Let $G(V,E)$ be a graph sampled from the $\sf{MCPC}$-block model w.r.t partition of $V$ into $V_{\ell,c}, (\ell,c) \in \{0,1\}^2$ where $k=\omega(\log n)$. Let $F(v)$
be the degree of the vertex. Then for 
any $v_i \in V_{\ell,1}$ we have 
$
F(v_i)= 2k + k\cdot (1\pm o(1)) (\sf{CC}_G(V_{\ell,1}))
$.
\end{theorem}

That is , the degree of a vertex is almost linearly related to the $\sf{CC}_G$ of the core it belongs in. If 
$\CC_G(V_{\ell,1})>\CC_G(V_{\ell',1})+C$ for any constant $C$, the degree of all vertices in $V_{\ell,1}$ will be higher than the ones in $V_{\ell',1}$. This will result in an \emph{unbalanced} ranking due to degree centrality. Next, we observe this in simulation.

{\bf Initial simulation.}
For simulation purposes, we instantiate $\mathbb{P}$ as in Table~\ref{tab:beta}, parameterized with $\gamma$. We set $k=20$ as it is a common choice for $K$-NN embedding, however the simulation results are similar for larger values of $k$. When $\gamma=0$, the generated graph will have approximately a $(0.3,0.25)$-$\sf{MCPC}$ structure, and $\sf{CC}_G(V_{0,1}) \approx \sf{CC}_G(V_{1,1})$ (i.e., the two cores have similar concentration values). 
In such a case, degree centrality, as well as the other centrality measures, have high balancedness throughout the ranking, as observed in Figure~\ref{fig: balancedness-beta0}. 
Moreover, these rankings have a high core prioritization (as the concentration of peripheries is lower). Thus, the induced graph of the top $c$-fraction of the points has a \emph{higher intra-community edge fraction (ICEF)} (which is simply the fraction of edges with endpoints being in the same community) as $c$ decreases, which we note in Figure~\ref{fig: accuracy-beta0}.


\begin{figure}
      \subfigure[\label{fig: balancedness-beta0} $\gamma=0$]{\includegraphics[scale=0.24]{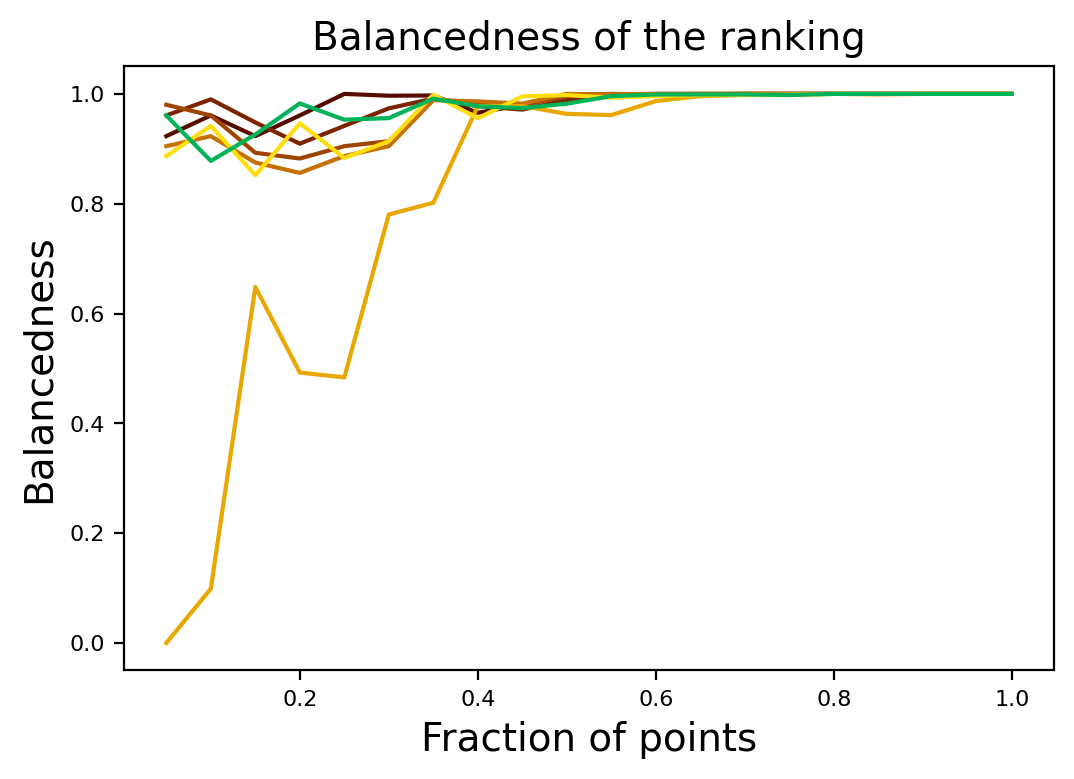}}
       \subfigure[\label{fig: accuracy-beta0}  $\gamma=0$]{\includegraphics[scale=0.24]{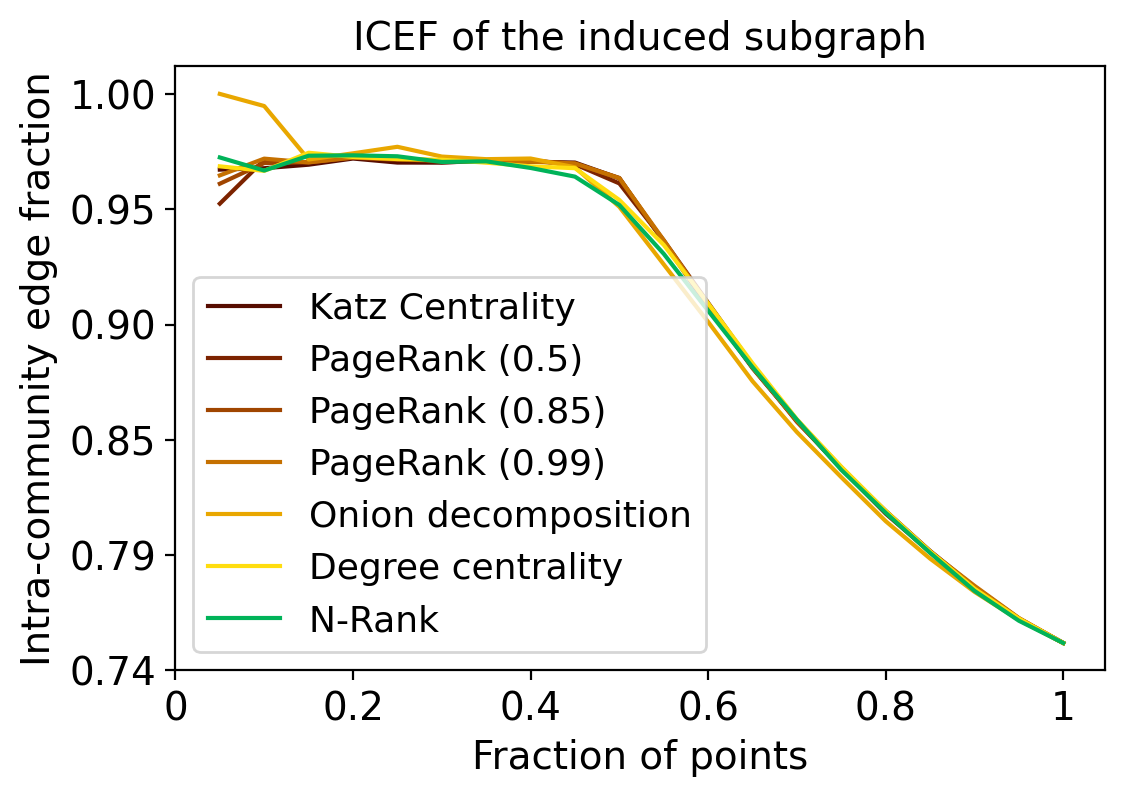}}
        \subfigure[\label{fig: balancedness-beta1}  $\gamma=0.05$]{\includegraphics[scale=0.24]{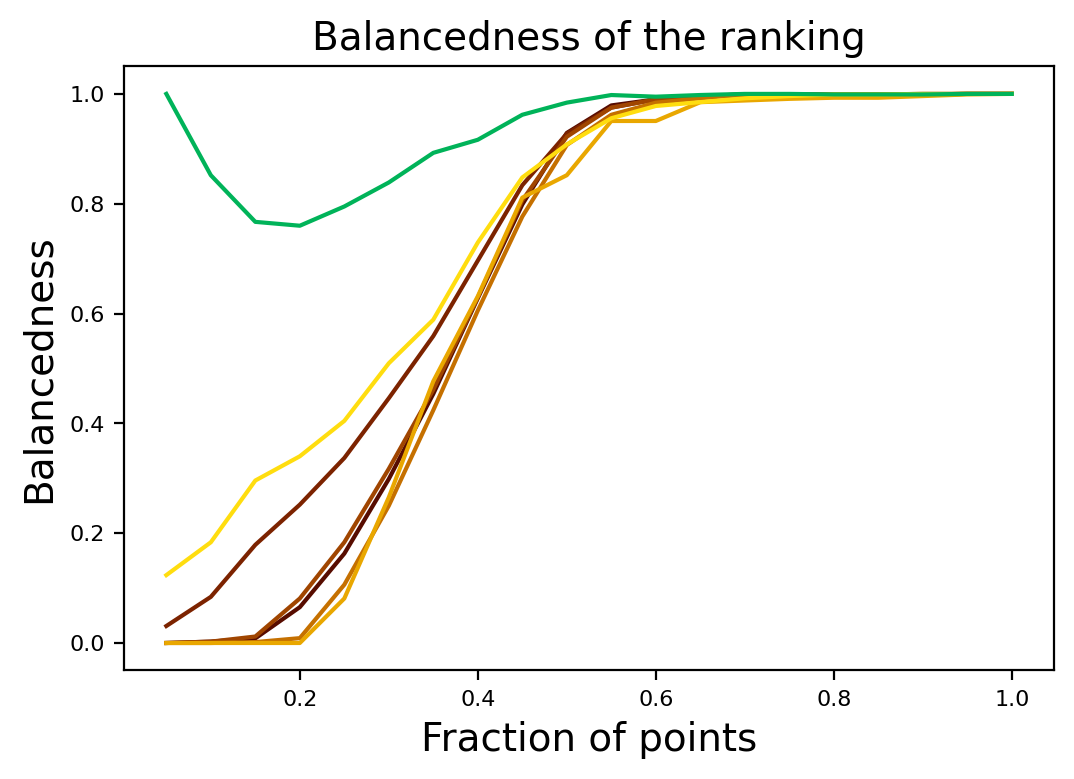}}
      \subfigure[\label{fig: accuracy-beta1}  $\gamma=0.05$]{\includegraphics[scale=0.24]{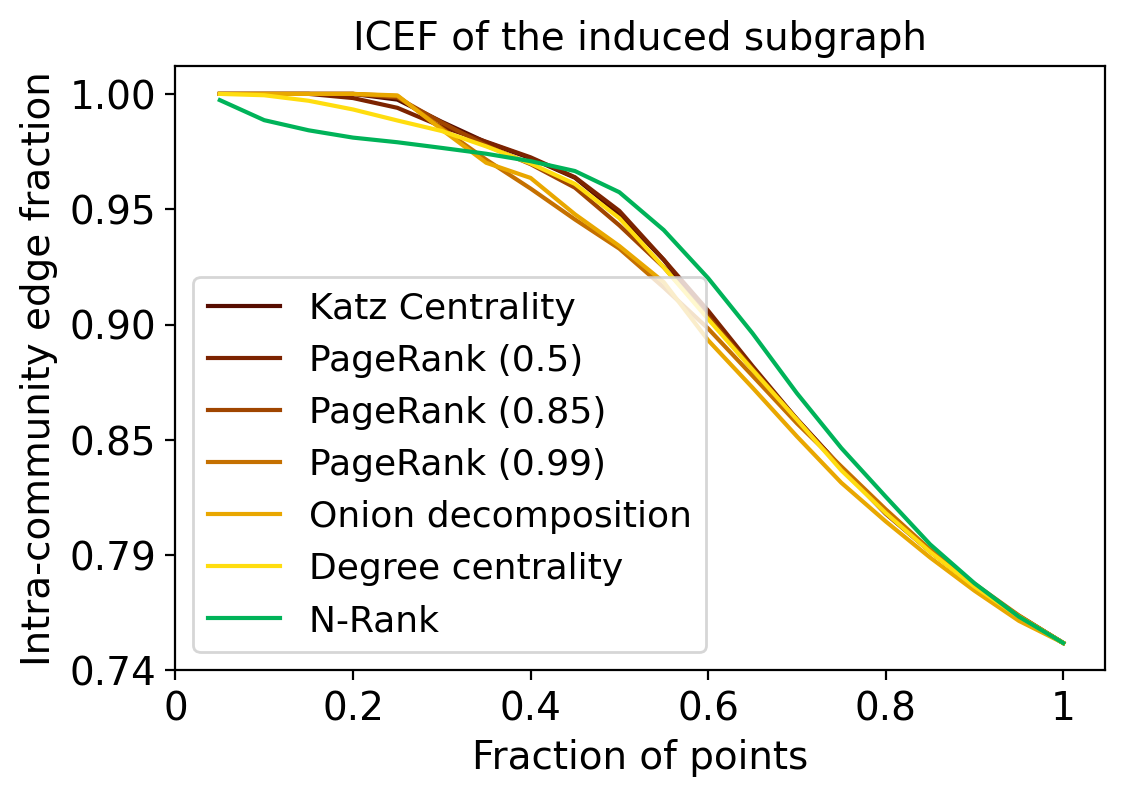}}
       \caption{Intra community edge fraction (ICEF) improvement and balancedness due to core-ranking}
\end{figure}

{\bf Unbalancedness when cores have varying concentration :} However, the scenario changes if $\gamma>0$. Consider
$\gamma=0.05$. Then we have $\sf{CC}_G(V_{1,1}) \approx \sf{CC}_G(V_{0,1})+0.07$. While the core prioritization will still be high for the centrality measures, the balancedness becomes very low, as shown in Figure~\ref{fig: balancedness-beta1}.

\subsection{Mitigating unbalancedness with relative centrality framework}
\label{sec: relative-centrality}

The primary reason behind the unbalancedness of the traditional centrality measures is that they capture the global centrality of vertices. To mitigate this issue, we aim to develop a method that generates a scoring with the following properties.

i) The core vertices should be assigned a higher score than the periphery vertices.

ii) The score of two vertices belonging to different cores should be similar, \emph{irrespective} of the cores they belong to. 

If a score on the vertices satisfies the two properties, we call it a \emph{relative centrality} measure. Then ranking vertices in descending order according to some score should lead to high balancedness as well as core prioritization. We first propose an initial Algorithm~\ref{alg:NgRank}, that we call N-Rank.

\adjustbox{varwidth=\linewidth,scale=0.8}{
\begin{algorithm}[H]
   \caption{NeighborRank (N-Rank) with $t$-step initialization}
   \label{alg:NgRank}
\begin{algorithmic}
   \STATE {\bfseries Input:} Graph $G(V,E),t$. Let the Adjacency matrix be $A$.
    \FOR{i in 1:n}
    \STATE $F^{(t)}(v_i) \gets \sum_{j}A^t_{j,i}$
     \hfill \COMMENT{\# Obtaining initial centrality score as sum of $i$-th column of $A^t$}
    \ENDFOR
     \FOR{$v_i \in V$}

    \STATE $S_{v_i} \gets \{v_j: v_j \in N_G(v_i), F^{(t)}(v_j)>F^{(t)}(v_i)\} \cup \{v_i\}$. \quad \quad \hfill \COMMENT{ $N_G(v_i)$ is the neighborhood of $v_i$}

    \STATE 
    
    \STATE $\hat{F}^{(t)}(v_i) \gets 
    \frac{F^{(t)}(v_i)}{\underset{v_j \in S_{v_i}}{\avg}       [F^{(t)}(v_j)]}$
    \quad \quad  \hfill  \COMMENT{ The relative-centrality step}
    \ENDFOR
    
    \RETURN $\hat{F}^{(t)}$

\end{algorithmic}
\end{algorithm}
}

{\bf Description of N-Rank.} 
Let the graph have an $(\alpha,\beta)$-$\sf{MCPC}$ structure.
First, we obtain an ``initial centrality measure'' for $v_i \in V$ as $F^{(t)}(v_i)$ as the sum of the $i$-th column of $A^t$, where $A$ is the adjacency matrix of the graph. For explanation, we consider $t=1$.

Then, for any $v_i \in V$ $F(v_i):=F^{(1)}(v_i)$ converges to the in-degree of $v_i$. Theorem~\ref{thm: deg-limitaion} states vertices from a core with higher $\sf{CC}_G$ will get a higher score. To mitigate this unbalancedness, for each $v_i \in V$, we select vertices $v_j \in N_G(v_i)$ ($v_i$'s neighborhood)
with \emph{higher $F$ score} and obtain the final score $\hat{F}(v_i)$ as the ratio of $F$ value of $v_i$ with average of these neighbors (including $v_i$ itself). Note that this value can be at most $1$.

Consider any vertex $v_i \in V_{\ell,1}$ where $G$ has an $(\alpha,\beta)$-$\sf{MCPC}$ structure. Furthermore, let $\beta$ be small ($o(1/k)$), i.e., inter-community edges between different cores are few. Then, for many such $v_i$, all of its neighbors will belong to either the same core, which will have a very similar score, or peripheries, which will have a lower score as they belong to a set with a lower concentration, following Theorem~\ref{thm: deg-limitaion}. Then, we have $\hat{F}(v_i) \approx 1$ for any $v_i \in V_{\ell,1}$, \emph{irrespective} of which core it is. We capture this behavior in Theorem~\ref{thm: main}.

\begin{theorem}[$1$-step N-Rank is good for the two-block model]
\label{thm: main}
    Let $G$ be a graph obtained from the $\sf{MCPC}$ block model with $k=\omega(\log n)$ resulting in an $(\Omega(1),o(1/k))-\sf{MCPC}$-structure w.r.t to the core-periphery blocks.
    Let $\hat{F}(v_i)$ be the score of the vertices $v_i, 1 \le i \le n$ as per Algorithm~\ref{alg:NgRank} for $t=1$. 
    Then, for $1-o(1)$ fraction of core vertices $v_i \in V_{\ell,1}$ for any $\ell$,  we have
 $1-o(1) \le \hat{F}_G(v_i) \le 1$.
\end{theorem}

In Lemma~\ref{lem: core-periphery-separation} we also show that the peripheries will have a lower score than the core vertices. Then, ranking the vertices in the descending order of $\hat{F}$ gives a \emph{balanced ranking with high core prioritization},
with similar ICEF improvement as the baselines (Figure~\ref{fig: accuracy-beta1}). Figure~\ref{fig: balancedness-beta1} shows that N-Rank has high balancedness for $\gamma=0.05$, even though the two cores have different concentrations.

Thus, N-Rank with $1$-step can be thought of as a way to create a relatively central version of degree centrality. 
Similarly, one could use a different centrality measure in the first step. Understanding the scope of using different methods as an initial centrality measure is an interesting direction.

\paragraph{Generalization into a meta-algorithm.}

Next, we generalize our algorithm in two natural ways.

1) There may be periphery vertices in the graph that have a high $F$ value compared to its $1$-hop neighborhood (E.g., a periphery vertex that has no edges going to a core). To mitigate this issue, we can look at look at some $p$-hop neighborhood $N_{G,p}(v_i)$ of $v_i$ when selecting the reference set.

2)    
We have observed that our N-Rank approach
increases the balancedness in the initial centrality measure $F$. In this direction, we can recursively apply this process by first calculating the $p$-hop N-Rank value and then feeding it back to the algorithm as the initial centrality measure to further increase balancedness. We can apply this recursive process any $q \ge 1$ many times.


\begin{figure}[t]
\begin{minipage}{0.32\linewidth}
    \begin{table}[H]
    \centering
    \scalebox{0.6}{
    \begin{tabular}{ccccc}
    \toprule 
        &$V_{0,1}$ & $V_{0,0}$ & $V_{1,0}$ & $V_{1,1}$ \\
                  \midrule
       $V_{0,1}$  & {\footnotesize $0.8+\gamma$} & $\frac{3(0.2-\gamma)}{8}$ & $\frac{3(0.2-\gamma)}{8}$ & $\frac{0.2-\gamma}{4}$ \\
       $V_{0,0}$  & {\footnotesize $0.4 +\gamma$} & $\frac{(0.6-\gamma)}{3}$ & $\frac{0.6-\gamma}{3}$ & $\frac{0.6-\gamma}{3}$\\
       $V_{1,0}$  & {\footnotesize $0.2$} & {\scriptsize $0.2$} &{\footnotesize $0.2$}  &{\footnotesize $0.6$} \\
       $V_{1,1}$  & $\frac{0.2+\gamma}{4}$ & $\frac{0.2+\gamma}{4}$ & $\frac{3(0.2+\gamma)}{8}$ & {\footnotesize $0.8-\gamma$} \\
       \bottomrule
    \end{tabular}
    }
    \caption{Parameterized block probabilities}
    \label{tab:beta}
    \end{table}
    \end{minipage}
    \begin{minipage}{0.67\linewidth}
        \begin{figure}[H]
          \subfigure[\label{fig-beta-balancedness} Core prioritization]{\includegraphics[scale=0.3]{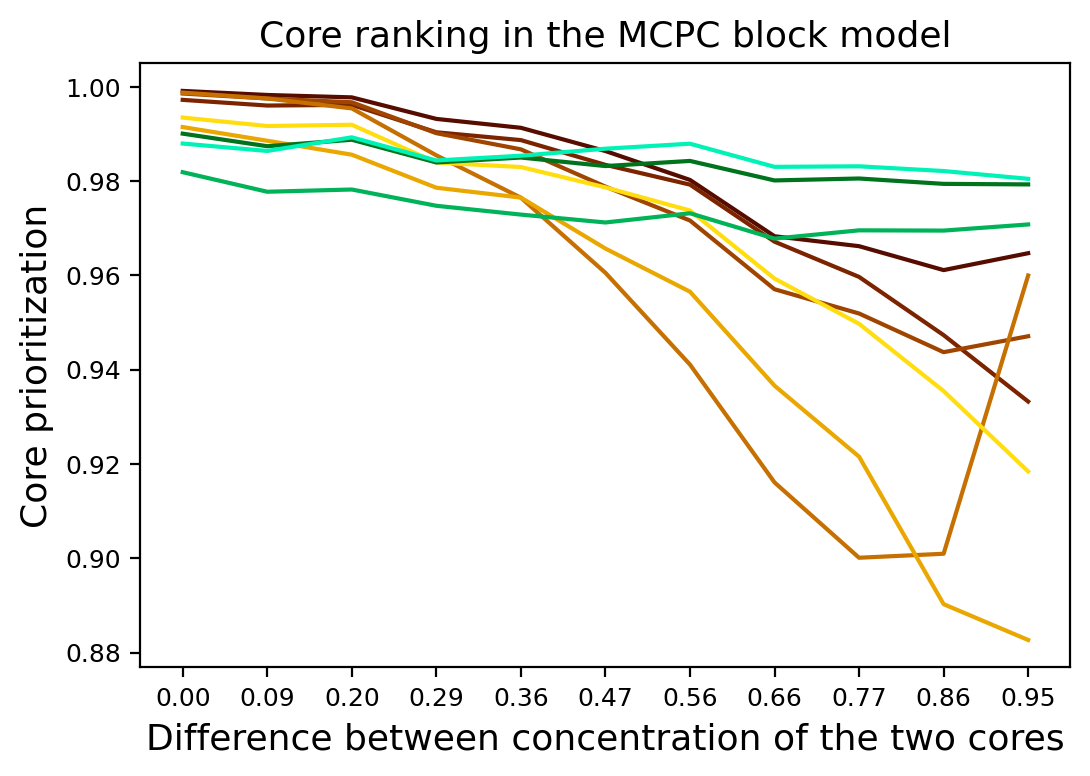}}
          \subfigure[\label{fig-beta-priori} Balancedness]{\includegraphics[scale=0.3]{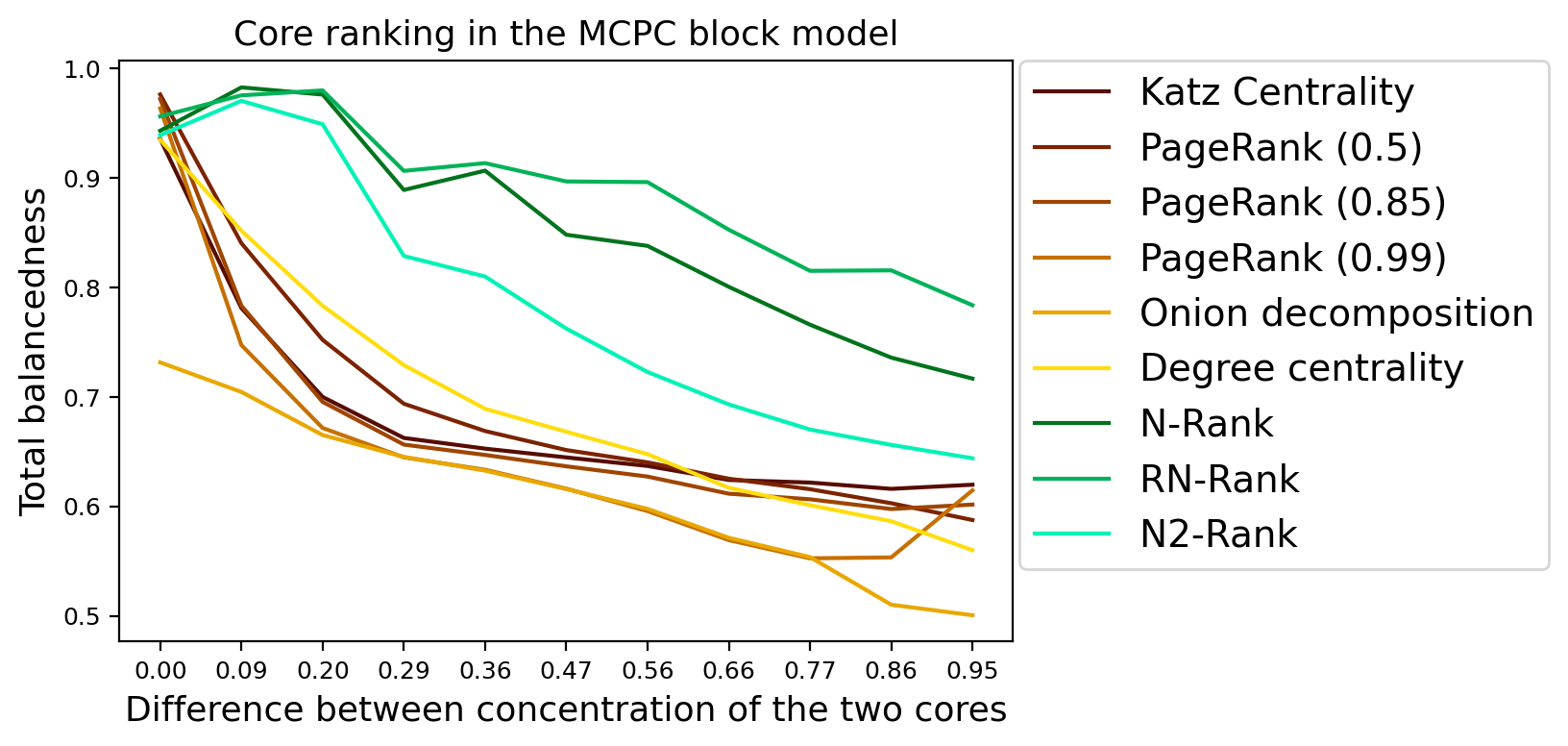}}
        \end{figure}
    \end{minipage}
 \caption{Core-ranking for graphs generated by Table~\ref{tab:beta} for different $\gamma$}
\end{figure}

We write the general method as a meta-algorithm M-Rank(t,p,q) (Meta-Rank)~\ref{alg:metaRank} in Appendix~\ref{app:meta-algorithm} along with an in-depth discussion. Then, N-Rank can be written as MR-Rank(1,1,0). We further define MR-Rank(t,2,0) as N2-Rank (t steps) and MR-Rank(t,1,1) as RN-Rank(t steps). Furthermore, we note that our method has a fast run time. 

For the initial centrality step, it is easy to see that the run time is $t \cdot |E|$. Then, the time to scan $N_{G,p}(v_i)$ for all vertices is
$\sum_{i=1}^n(N_{G,p}(v_i))$. 
Finally, applying the centralizing step $q$-times would take 
time $q \cdot (\sum_{i=1}^n(N_{G,p}(v_i)))$. For, $k$-regular graphs, 
we have $\sum_{i=1}^n(N_{G,p}(v_i)) \le k^p|V|$.

Thus for $k$-regular graphs, the time complexity of M-Rank(t,p,q)
is $\OO(t \cdot |E|+ q \cdot k^p \cdot |V|$). That is, it is 
$\OO(t|E|+k|V|)$ for N-Rank, $\OO(t|E|+2k|V|)$ for RN-Rank and $\OO(t|E|+k^2|V|)$ for N2-Rank which is considerably fast for small values of $t$.


\paragraph{Large scale simulation and core-prioritization vs. balancedness tradeoffs}
We use the $\gamma$-parameterized block probability matrix in Table~\ref{tab:beta} and vary $\gamma$ from $0$ to $0.2$ with an increment of $0.02$, generating a graph for each value for a large-scale simulation. We make the following observations. 

i) As $\gamma$ increases, three structural changes occur in the graph. $\alpha$ (separation of $\sf{CC}$ values between cores and peripheries) decreases, 
$\beta$ (inter-core edge fraction) increases, and the difference between the core concentration values of the two cores becomes larger.

ii) We apply our core-ranking algorithms N-Rank, N2-Rank, and RN-Rank, all with $t=1$, and compare them with the baseline algorithms. We plot the core prioritization and balancedness of the ranking in 
Figures~\ref{fig-beta-balancedness} and \ref{fig-beta-priori} respectively. We observe that as $\gamma$ increases, our methods have higher balancedness than traditional centrality measures (with RN-Rank, particularly, having very high balancedness) while having comparable core prioritization with the other methods. In contrast, N2-Rank has a slightly higher core prioritization than our other methods. This points to an interesting tradeoff between different instantiations of our meta-algorithm.



\section{\texorpdfstring{$\sf{MCPC}$}~  structures in the \texorpdfstring{$K$}~-NN embeddings of vector datasets}
\label{sec:K-NN-embedding}

In this Section, we explore the presence of $\sf{MCPC}$ structure in $k$-NN embedding of vector datasets with underlying community structure and the impact of core-ranking algorithms. 
Converting a vector dataset into a graph via nearest neighbor like embedding and then inference on that graph is a common pipeline in unsupervised learning. For example, this is the first step in the popular dimensionality reduction and visualization algorithm UMAP~\cite{UMAP}. Similarly, one of the state of the art unsupervised algorithm for single-cell datasets, known as Seurat~\cite{seurat}, first obtains the $K$-NN embedding, and then converts it into a shared-nearest neighbor embedding, on which it applies graph clustering algorithms like Louvain. In this paper, we focus on the $k$-NN embedding itself of vector datasets with underlying community structures, and then observe that in many natural and real-world cases
such graphs can have an $\sf{MCPC}$ structure. 

For the sake of uniformity, we set $k=20$ for all of the experiments from here on. We note that this is a widely chosen value for nearest neighbor embedding in many literature. Small changes in the value of $k$ does not have a significant impact on our observations.

\subsection{Concentric GMM}
\label{sec: C-GMM}

We simulate this phenomenon using mixture models. 
Here, each underlying community has a center, and then points for that community are generated with respect to distributions with mean as the respective center. When the distributions are Gaussian, the model is called the Gaussian mixture model (GMM)~\cite{mixture-model-1-GMM}, which is widely studied in the clustering literature~\cite{GMM-clustering-3,mixture-clustering-4,GMM-clustering-1,GMM-clustering-2}. In this section, we extend the GMM in a natural way to incorporate an $\sf{MCPC}$ structure.

\begin{definition}[Concentric GMM with two communities]
There are two centers $\mathbf{c_{\ell}}\in \mathbb{R}^d, \ell \in \{0,1\}$. \emph{Each center} is associated with \emph{two} $d$-dimensional isotropic Gaussian distributions with the center as its mean and variances $\sigma_{\ell,1}$ and $\sigma_{\ell,0}$ respectively, with 
$\sigma_{\ell,0} \ge 1.1\sigma_{\ell,1}$.
We denote the distributions as $\mathcal{D}^{(\ell)}_1$ and $\mathcal{D}^{(\ell)}_0$ respectively.
We sample $n_{\ell,j}$ points from each distribution $\mathcal{D}^{(\ell)}_j$, and collectively note them as
$V_{\ell,j}$.
Then, the two underlying communities are $V_{\ell}= V_{\ell,1} \cup V_{\ell,0}$. 
\end{definition}

For simulation, we set $d=20$, $|V_{\ell,c}|=2000$ for all $(\ell,c)$ and the variances of the distributions are set as~
$\begin{bmatrix}
\sigma_{0,1} & \sigma_{0,0} & \sigma_{1,1} & \sigma_{1,0}\\
0.1 & 0.3 & \gamma\cdot 0.1 & \gamma \cdot 0.3
\end{bmatrix}$ where $1 \le \gamma <3$. Furthermore, we choose the centers as two so that there is an overlap between the two communities. Then, we are interested in the $20$-NN embedding of the resultant points.

i) The 20-NN embedding has an $(\alpha,\beta)-\sf{MCPC}$ periphery structure, with the points sampled from the lower variance distribution $\sigma_{0,1}$ and $\sigma_{1,1}$ forming the cores and the rest being peripheries. 

ii) When $\gamma=1$, the two communities are symmetrical and we have a $(0.7,0.3)-\sf{MCPC}$ structure with both cores having similar concentration. Then all core-ranking methods have high balancedness, as noted in Figure~\ref{fig:GMM-bal-bval}, and when we select the top points from the ranking, the induced subgraph has a higher ICEF, as observed in Figure~\ref{fig:GMM-bal-acc}. From hereon, we set $t:=\log n$ step in the first step for all our algorithms to avoid local maxima when obtaining the initial centrality scores.

iii) \textbf{Cores with different concentration:}
Next, we observe that if we set $\gamma=1.5$, then one core has a higher variance than the other, and the corresponding vertices in the graph have a lower concentration than the other core. Then, traditional centrality measures indeed have a lower balancedness, and our methods perform relatively better, with RN-Rank having the highest balancedness value, as noted in Figure~\ref{fig:GMM-unbal-bval}. Furthermore, we note that all methods provide a similar increase in ICEF upon core selection in Figure~\ref{fig:GMM-unbal-acc}, with N2-Rank being slightly better than our other methods, the same as in the random graph model. We provide a more detailed simulation result for different values of $\gamma$ in Appendix~\ref{app: conc-GMM} to further support the balancedness of relative centrality.


\begin{figure}[t]
\hspace{-4em}
\resizebox{1.2\linewidth}{!}{
     \subfigure[\label{fig:GMM-bal-acc} Cores with same $\sf{CC}_G$]{\includegraphics[scale=0.235]{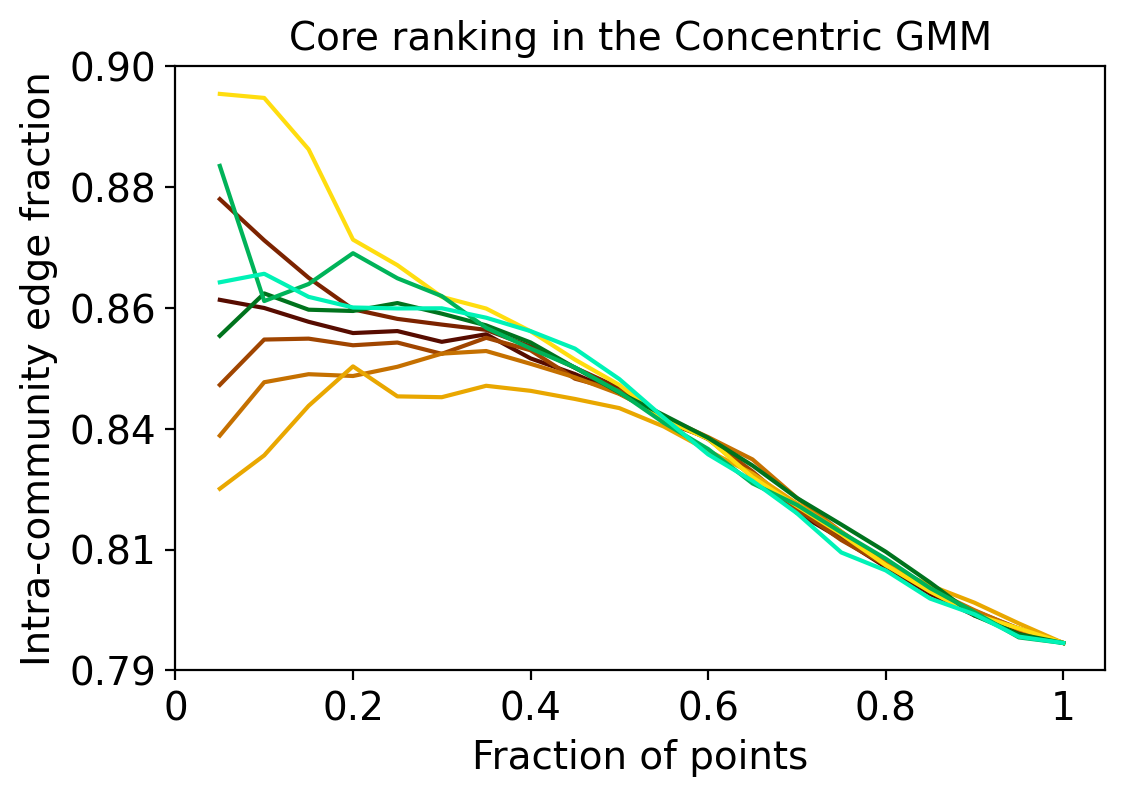}}
  \subfigure[\label{fig:GMM-bal-bval} Cores with same $\sf{CC}_G$]{\includegraphics[scale=0.235]{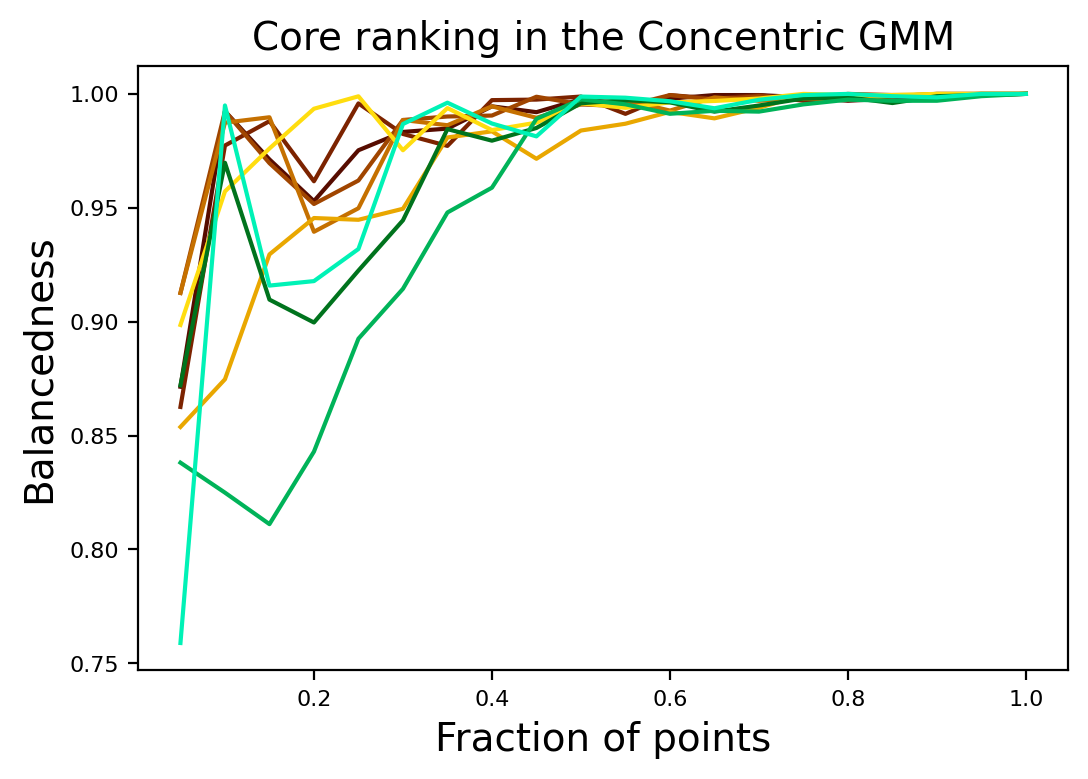}}
    \subfigure[\label{fig:GMM-unbal-acc} Cores with different $\sf{CC}_G$
    ]{\includegraphics[scale=0.235]{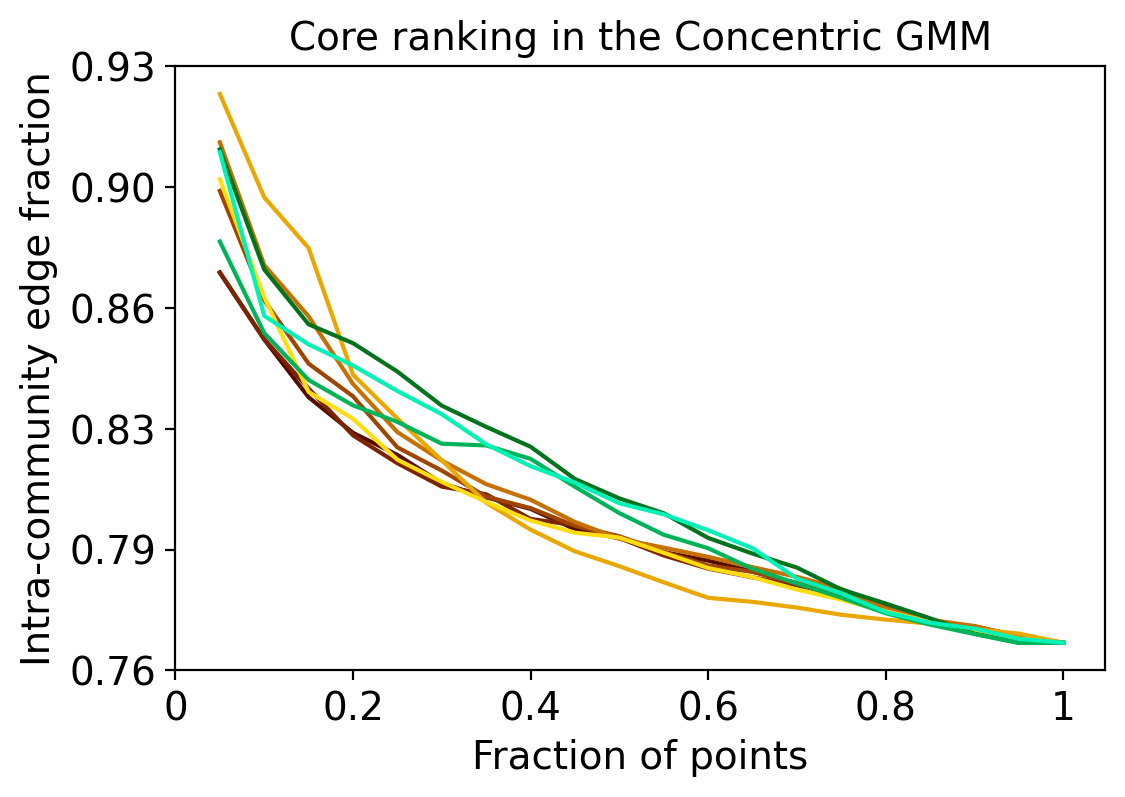}}
  \subfigure[\label{fig:GMM-unbal-bval} Cores with different $\sf{CC}_G$]{\includegraphics[scale=0.235]{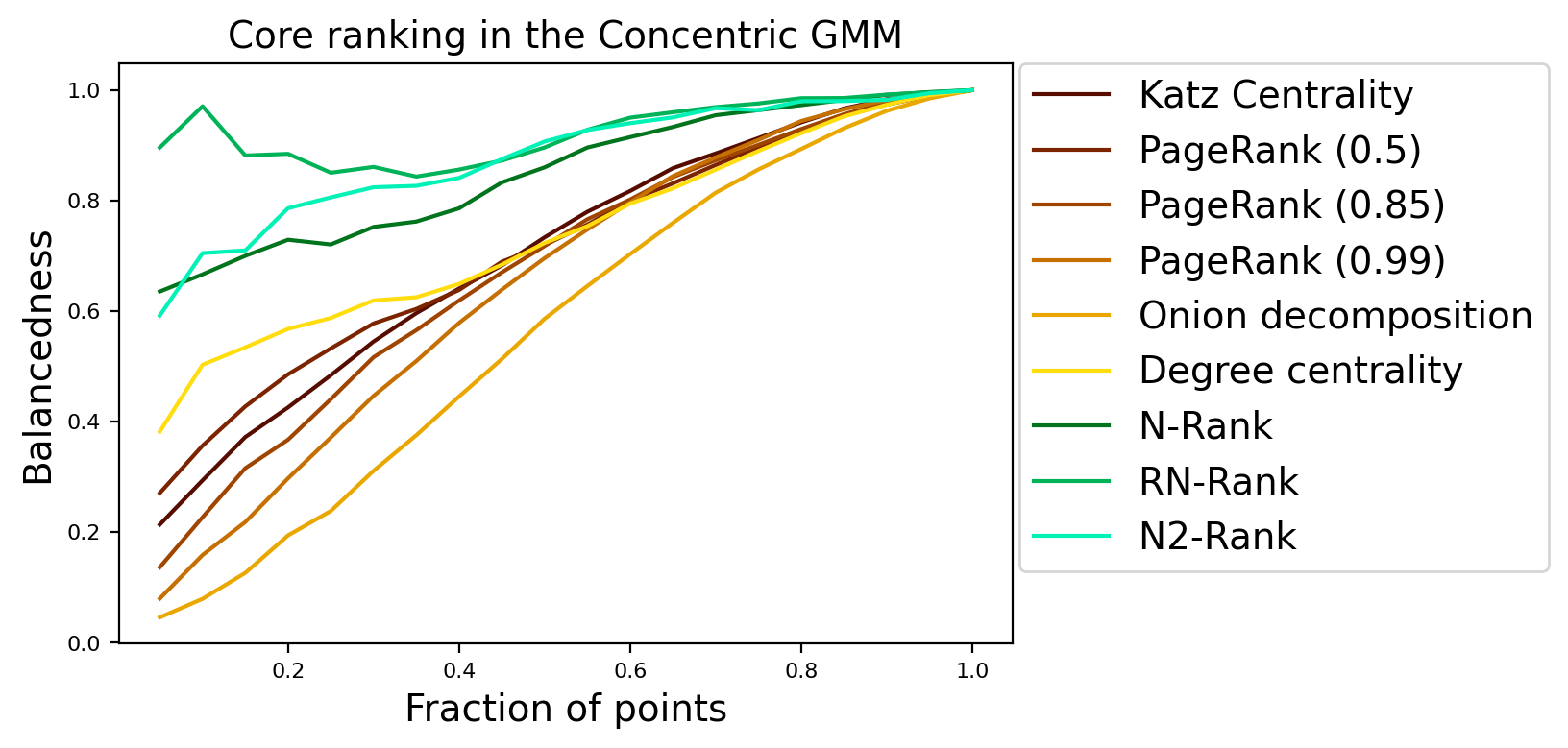}}
  }
  \caption{Improvement in ICEF, and balanced of core-ranking in concentric GMM}
\end{figure}

\subsection{Single-cell data}
\label{sec:single-cell}
Finally, in this section, we apply core ranking algorithms to single-cell RNAseq datasets, which we simply refer to as single-cell datasets. We work in the following setup.

{\bf Datasets:} 
We use a total of $11$ datasets.
We use the $7$ datasets from a recent database~\cite{single-cell-7data},
the popular Zheng8eq dataset~\cite{single-cell-duo},
and two more large 
datasets~\cite{single-cell-ALM-VISP}, and a T-cell dataset~\cite{nature-medicine-t-cell} of cancer patients (which are considered harder datasets). All of these datasets have annotated labels indicating the cell types of the data points, which form  the underlying (ground truth) communities. The size of the datasets varies from $1,400$ to $54,000$. For each dataset, we first pre-process it with a standard pipeline (details in Appendix~\ref{app:single-cell}) and then obtain its $20$-NN graph embedding, which we denote as $G_0$. Then, we apply the baseline and our core-ranking algorithms to these graphs. We make the following observations. 

{\bf Inference:}
For each graph $G_0(V,E)$, we select some top $c$-fraction of the points in the descending order of their ranking as per the core-ranking algorithm $F$ and denote it as $F_c(V)$. We first explain the observations w.r.t to the T-cell dataset.

i) {\bf Core-ranking improves community structure:} 
We observe that as $c$-decreases, the corresponding induced subgraph has a \emph{higher intra-community edge fraction} (ICEF), as observed in Figure~\ref{fig:tcell-acc}. 

ii) {\bf The selected cores have higher concentration:} Similarly, Let $S_c$ be the top $c$-fraction of the points. Then, we have the very interesting observation that 
$\sf{CC}_{G_0}(V_{\ell}\cap S_c)> CC_{G_0}(V_{\ell})$. That is, the higher-ranked points indeed form the cores of their communities as per our Definition~\ref{def: MCPC}. As support,we present the values for $c=0.2$ for RN-Rank in Table~\ref{fig:tcell-corev}. This, together with the previous point, implies that the single cell data indeed has $\sf{MCPC}$ structure w.r.t the ground truth communities.

iii) {\bf Relative centrality has a higher balancedness:}
 The single-cell datasets we look at have many ground truth communities (sub-populations), and we aim to keep points from all communities at the top of the ranking. Thus, balancedness, as defined in Definition~\ref{CP:metric}, is unsuitable, as it only captures the behavior of the worst-preserved community. In this direction, we define the following metric.

\begin{definition}[Preservation ratio]
Given a set of points $V$ with an underlying partition $V_1, \ldots ,V_{z}$ and a subset $S \subset V$, the preservation ratio of $S$ w.r.t the underlying partition is defined as 
$
PR(S,V) = \frac{|V|}{z|S|} \cdot \sum_{\ell=1}^{z} \min \left\{ \frac{|S \cap V_{\ell}|}{|V_{\ell}|}, \frac{|S|}{|V|} \right \}
$.
\end{definition}
That is, each ground truth cluster contributes the minimum of $\frac{|S \cap V_{\ell}|}{|V \cap V_{\ell}|}$,
$\frac{|S|}{|V|}$ to the term. We want to observe what fraction of points selected in the core are necessary to achieve proportionality. The higher this value, the more clusters have a $|S|/|V|$ fraction of the points in the filtered set. Note that this definition penalizes the fraction by which the points selected from a community is lower than the $|S|/|V|$ fraction. Note that the value of $PR(V,S)$ lies between $1/z$ and $1$. Furthermore, when there are only two communities, any set's preservation and balancedness values are related by a fixed linear equation. 

In this direction, we observe that our methods have a \emph{superior preservation ratio} throughout for the T-cell dataset in Figure~\ref{fig:tcell-preserve}, with RN-Rank having the best performance. In fact, at the $c=0.2$, the baseline methods \emph{completely miss} to include vertices from some communities, which we succeed at. The plots of all of the other datasets, along with a summarization, can be found in Appendix~\ref{app:single-cell}.




\paragraph{Improvement in ICEF and clustering outcomes.} 
Finally, we observe that not only does the induced graph by $F_c(V)$ have higher ICEF, but the subgraph is also better separable into its ground truth communities. To this end, we set $c=0.2$ (this choice is arbitrary). Then, we apply the well-known Louvain algorithm~~\cite{louvain} on the original 20-NN embedding as well as the induced subgraph for each CR algorithm. We compare the purity of the outcome on these points compared to the whole graph. To present a complete picture, we also show the original ICEF, the intra-community edge fraction of the top $c$-fraction of points, as well as the preservation ratio of these points. We place the results in Table~\ref{tab:single-cell-preserve-purity}.
We show the results of 8/11 datasets here.
We observe a slight improvement in the ICEF for the first $2$  datasets (datasets with high initial ICEF). The improvement is more significant for the $6$ harder datasets. As before, all methods have similar improvements in ICEF, and also purity. Our relative centrality based methods demonstrate a superior preservation ratio, with \emph{RN-Rank being the best}. Furthermore, as we have previously observed, N2-Rank has slightly higher ICEF improvement but with a lower preservation ratio, further underlying the tradeoff between different instantiations of our Meta-Rank algorithm.
We present several other experiments in Appendix~\ref{app:single-cell}, which includes the purity comparison of all datasets, as well as the NMI improvement comparisons of all of the datasets for $c=0.2$.

\begin{figure}[t]
  \subfigure[\label{fig:tcell-acc}  ICEF of induced subgraph]{\includegraphics[scale=0.38]{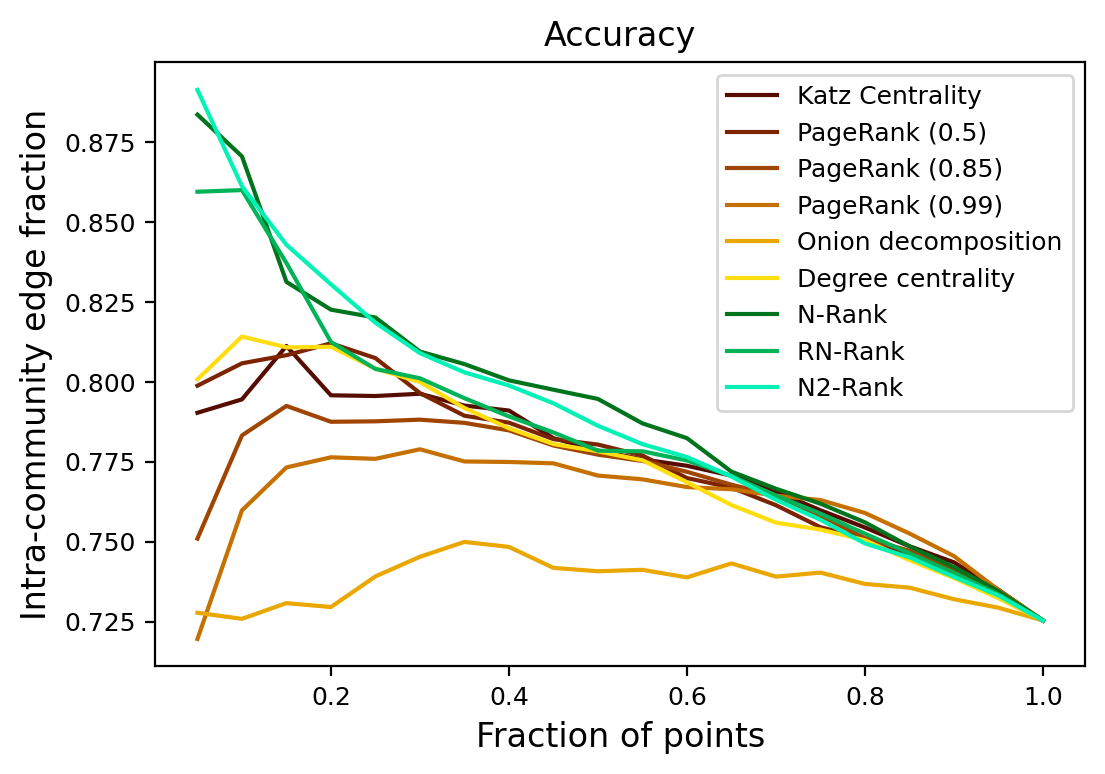}}
  \subfigure[\label{fig:tcell-preserve} Preservation ratio of the subset]{\includegraphics[scale=0.38]{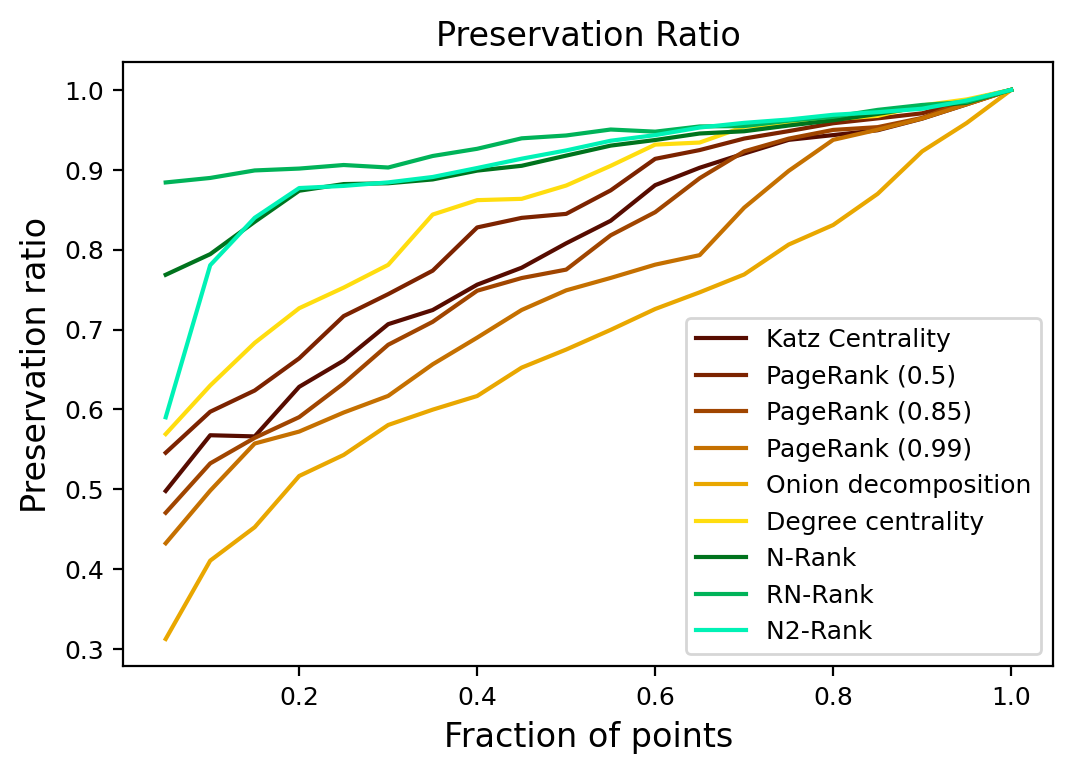}}
  \subfigure[\label{fig:tcell-corev}$\sf{CC}$ of whole community vs. selected core]{\raisebox{1.5\height }{\mytab}}
\caption{Improvement in intra-community accuracy and balancedness by different ranking algorithms}
\label{fig:tcell-main}
\end{figure}

\begin{table}[ht]
    \centering
    \hspace*{-.7in}
    \resizebox{1.2\linewidth}{!}{
    \begin{tabular}{cccccccccccc}
      Datasets   &  BH & Se &  Zh & Tcell  & ALM  & AMB &TM &VISP \\
      \# of points &  1886 & 2133 & 3994 & 5759  & 10068  & 12832 &54865 &15413 \\ 
    Metrics &  PR ~ICEF ~ Purity& PR ~ICEF ~ Purity& PR ~ICEF ~ Purity&PR ~ICEF ~ Purity&PR ~ICEF ~ Purity&PR ~ICEF ~ Purity&PR ~ICEF ~ Purity&PR ~ICEF ~ Purity\\
    \toprule
     \makecell{Original \\ values}   &  1.00 0.94 0.93  & 1.00 0.92 0.89 &  1.00 0.78 0.8  & 1.00 0.72 0.72  &  1.00 0.68 0.44  &  1.0. 0.74 0.46  &  1.00 0.94 0.86 & 1.00 0.69 0.48 \\
    \midrule

    \textcolor{darkolivegreen}{RN-Rank}  &     {\bf 0.68} 1.00 0.99  &  {\bf 0.56} 0.99 0.99  &   0.85 0.88 {\bf 0.89}  &  {\bf 0.89} 0.81 0.79  &  {\bf 0.61} 0.81 0.71  &  {\bf 0.69} 0.83 0.69  &  {\bf 0.87} 0.98 0.95  &  {\bf 0.61} 0.79 0.68  \\  

    \textcolor{darkgreen}{N2-Rank} &   0.50 1.00 1.00  &    0.52 0.99 1.00  &   0.82 0.89 {\bf 0.89}  &  0.87 0.82 {\bf 0.85}  &  0.51 0.81 0.71  &  0.63 0.84 0.73  &  0.85 0.98 0.96  &  0.56 0.80 0.69  \\
    \midrule
    Katz &     0.48 1.00 1.00  &   0.50 1.00 1.00  &   0.90 0.79 0.83  &  0.60 0.80 {\bf 0.85}  &  0.43 0.83 0.73  &  0.51 0.87 0.78  &  0.74 0.99 0.98  &  0.49 0.82 0.73  \\  

PageRank (0.5) &   0.51 1.00 1.00  &  0.50 0.99 1.00  &  0.93 0.78 0.82  &  0.69 0.80 0.84  &  0.47 0.82 0.72  &  0.57 0.86 0.76  &  0.76 0.99 0.98  &  0.53 0.82 0.71  \\  

PageRank (0.85) &  0.49 1.00 1.00  &   0.51 0.99 1.00  &  0.88 0.79 0.79  &  0.60 0.78 0.84  &  0.42 0.84 0.74  &  0.49 0.87 0.79  &  0.73 0.99 0.98  &  0.47 0.83 0.72  \\  

PageRank (0.99) &    0.45 1.00 1.00  &   0.50 0.99 1.00  &  0.83 0.79 0.81  &  0.57 0.78 0.85  &  0.40 0.84 {\bf 0.76}  &  0.46 0.87 { 0.80}  &  0.72 0.99 0.99  &  0.45 0.83 0.73  \\  

Onion &   0.34 1.00 1.00  &   0.22 0.99 0.98  &   0.92 0.75 0.75  &  0.50 0.71 0.77  &  0.24 0.83 0.73  &  0.35 0.88 {\bf 0.82}  &  0.45 0.98 {\bf 0.97}  &  0.35 0.82 {\bf 0.75} \\  

Degree &   0.52 1.00 1.00  &  0.50 0.99 1.00  &   {\bf 0.96} 0.79 0.81  &  0.73 0.80 0.81  &  0.52 0.80 0.69  &  0.61 0.85 0.71  &  0.78 0.98 0.96  &  0.57 0.81 0.67  \\
    \bottomrule
    \end{tabular}
    }
    \caption{Preservation ratio, ICEF, and purity score of Louvain of top $20\%$ ranked points}
    \label{tab:single-cell-preserve-purity}
\end{table}

\paragraph{Summarization.}
Thus, our relative centrality measures obtain a more balanced ranking (which we capture with the preservation ratio) in a large set of real-world datasets compared to traditional centrality measures while having a comparable improvement in the structural accuracy (captured by ICEF) as well as improvement in the performance of downstream clustering algorithms (captured by the purity of Louvain outcome of the selected cores). Now, we conclude our paper with some more discussions.

\section{Conclusion}
\label{sec: conclusion}

In this paper, we have explored the coexistence of community structures and core-periphery structures in graphs with the formalization of $\sf{MCPC}$, and devised balanced core-ranking algorithms for graphs with this structure via novel concept that we term as relative centrality. 

We believe that this is only an initial step towards understanding such structure in real world datasets, and posit several open questions that we think are important, and of interest. 
\begin{enumerate}

    \item Our proofs are in the MCPC-block model, which we use to prove the unbalancedness of traditional centrality measures and the performance of our algorithm. It will be interesting to obtain similar proofs in the other simulation model we study, the concentric GMM model. Furthermore, our proofs can be extended to the case when the size of the cores and peripheries are different. Finally, we note that while we showed the results for $k=\omega(\log n)$, but it can be modified to the case of $k \ge 2 \log n$.

    \item We proposed the algorithm concept of relative centrality and proposed the meta-algorithm Meta-Rank in Algorithm~\ref{alg:metaRank}. We used two instantiations of our algorithm, namely N2-Rank and RN-Rank and showed interesting ICEF improvement vs. balancedness tradeoffs. However, obtaining the best possible instantiations for different problems remains an interesting open direction. 

    \item Even though our methods provide a superior preservation ratio compared to the traditional centrality measures in the single-cell data, we also have a lower-than-ideal preservation ratio in some datasets, such as the Segerstolpe dataset. Designing algorithms with even higher balancedness is an important goal.

    \item In this paper, our real-world experiments focused on a large set of single-cell datasets. We believe \sf{MCPC} structures may be present in other kinds of datasets as well, and we are in the process of capturing more examples.

    \item It is well known that complex real-world vector datasets are known to have non-linear geometry, which is often referred to as the manifold hypothesis (or more recently, the union of manifold hypothesis). In this direction, a characterization of non-linear vector datasets such that their graph embeddings have $\sf{MCPC}$ structures will further enhance our understanding of the underlying geometries.

    \item In this paper, we focused on \emph{unweighted} graphs. Exploration of weighted NN embedding (with different weight functions) and other kinds of embedding, such as heat-kernel embedding, is an important future work.

    \item Finally, we note that in this paper, we used our methods to obtain a subset of the dataset that is better separable into the ground truth communities. An important next step, and thus a limitation of this work, is that we do not use the clustering on this set to obtain a better clustering of the whole dataset. While we have some progress in this direction, it is beyond the scope of the paper and requires a systematic study.

\end{enumerate}

\bibliography{reference-MCPC}
\bibliographystyle{alpha}

\appendix
\newpage
\section{Structures and results in the \texorpdfstring{$\sf{MCPC}$}~-block model}
\label{app: proof}

In this section, we provide theoretical support to the observed unbalancedness of centrality measures as well as the balancedness and efficacy of relative centrality as discussed in Section~\ref{sec:graph-setting} in the random graph model. We first  reintroduce the model for ease of following.

\paragraph{The generative block model}
We are interested in the random graph generated by a $4$-block model. In this model, we are given a set of $n$ vertices $V=\{v_1, \ldots , v_n\}$ that has a partition into communities $V_0, V_1$, and each community $V_{\ell}$ has a further partition into a core $V_{\ell,1}$ and the periphery $V_{\ell,0}$. We focus on the balanced case, where $|V_{\ell,c}|=n/4$ for any $(\ell,c)$ pair.  
This is associated with a $4 \times 4$ block-probability matrix $\mathbb{P}$, that is indexed with $(i,j) \in \{0,1\}^2$. Furthermore, $\mathbb{P}$ is row stochastic, i.e,
$\sum_{(\ell',c')} \mathbb{P}[(\ell,c),(\ell',c')]=1$ for all $(i,j)$ pair. 

Then, for each vertex pair $v_i \in V_{\ell,c}$ and $v_j \in V_{\ell,c'}$, 
we add an $v_i \rightarrow v_j$ edge with probability $\mathbb{P}[(\ell,c),(\ell',c')]\cdot \frac{k}{|V_{\ell',c'}|}$.

Then, we work under the following setting.

\begin{enumerate}
    \item $\mathbb{P}[(\ell,c'),(\ell,1
)]$ is $\Omega(1)$ for all $\ell$ and any $c'$. That is, within a community, a constant fraction of edges going out from a periphery vertex ends up in the corresponding core, and a constant fraction of edges originating in a core vertex remains in the core. 

    \item $k=\omega(\log n)$.

    \item We denote the total degree, in-degree, and out-degree of a vertex $v_i$ with $\deg(v_i)$, $\deg_{+}(v_i)$ and $\deg_{-}(v)$ respectively.

\end{enumerate}
We denote the resultant graph-generating process as 
$\sf{BM}(\{V_{\ell,c}\}_{(\ell,c) \in \{0,1\}^2},\mathbb{P},k)$

Then, we are interested in the case when $\sf{BM}$ generates an $(\alpha,\beta)-\sf{MCPC}$ structure with $\alpha=\Omega(1)$ and $\beta \ll 1$. Recall that this implies that $\sf{CC}_G(V_{i,1})>\sf{CC}_G(V_{i,0}) +\alpha$ for $i \in \{0,1\}$ and the fraction of the inter-community edges originating in the core is $\beta$ fraction of that of the ones starting in the periphery.

\paragraph{Preliminary bounds}
We are primarily interested in the degree of the vertices generated by such a graph to quantify the unbalancedness in degree centrality as well as the power of our \emph{relative centrality} approach. We define the map $X:[n] \rightarrow \{0,1\}^2$ s.t  $X[i]=(\ell,c) \iff v_i \in V_{\ell,c}$. First we note down the Chernoff concentration bound, which will be useful going forward.

\begin{theorem}[Chernoff Hoeffding bound~\cite{Chernoff}]
\label{thm: chernoff}
Let $z_1, \ldots , z_n$ be i.i.d random variables that can take values in $\{0,1\}$, with $\mathbb{E}[z_i]=p$ for $1 \le i \le n$. 
Then we have 
\begin{enumerate}
   \item $\Pr \left( \frac{1}{n}\sum_{i=1}^n z_i \ge p+ \epsilon \right) \le e^{-D(p+\epsilon||p)n} $    

   \item $\Pr \left( \frac{1}{n}\sum_{i=1}^n z_i \le p - \epsilon \right) \le e^{-D(p-\epsilon||p)n} $
   
\end{enumerate}

where $D(x||y)$ is the KL divergence of $x$ and $y$.
\end{theorem}

Next, we obtain some bounds on out-degree and in-degree of vertices.

\begin{lemma}[The graph is almost-regular w.r.t out-degree]
\label{lemma:out-degree}
Let $G$ be a graph generated from $\sf{BM}(\{V_{\ell,c}\}_{(\ell,c) \in \{0,1\}^2},\mathbb{P},k)$. Then, with probability $1-n^{-4}$, the out degree of any vertex $v_i \in V$ is bounded as
\[
k-o(k) \le \deg_{-}(v_i) \le k+o(k)
\]
\end{lemma}
\begin{proof}
Let $e_{i,j}$ denote the indicator random variable indicating which is 1 if there is an edge from $v_i$ to $v_j$, and $0$ otherwise. 

Then, $\deg_{-}(v_i)=\sum_{j \in [n]} e_{i,j}$. 
Furthermore, let $\deg_{(-,S)}(v_i)$ denote
$\sum_{j: v_j \in S} e_{i,j}$. 
Then, from a simple counting argument, we have  $\mathbb{E}[\deg_{(-,V_{\ell,c})}(v_i)]= \frac{k}{|V_{\ell,c}|} \cdot |V_{\ell,c}| \cdot \mathbb{P}[X[i],(\ell,c)]= \mathbb{P}[X[i],(\ell,c)] \cdot k$.

Next, by Chernoff bound we have
\begin{align*}
\Pr(\deg_{-}^{V_{\ell,c}}(v_i)> \mathbb{E}[\deg_{-}^{V_{\ell,c}}(v_i)]+ 8\cdot \sqrt{k}\sqrt{\log n}) 
\le e^{-D\big(\mathbb{E}[\deg_{-}^{V_{\ell,c}}(v_i)]/|V_{\ell,c}|+ 8\cdot \sqrt{k}\sqrt{\log n}/|V_{\ell,c}| \big|\big| \mathbb{E}[\deg_{-}^{V_{\ell,c}}(v_i)]/|V_{\ell,c}| \big) \cdot |V_{\ell,c}|}
\end{align*}

We note that the KL divergence $x,y$ is
$D(x||y)=x \ln(x/y) + (1-x)\ln((1-x)/(1-y))$, which is upper bounded by $(x-y)^2/2x$ when $x>y$. Then, 
the KL divergence can be upper bounded as 
\begin{align*}
&
D \bigg(\mathbb{E}[\deg_{-}^{V_{\ell,c}}(v_i)]/|V_{\ell,c}|+8\cdot \sqrt{k}\sqrt{\log n}/|V_{\ell,c}| \bigg|\bigg|\mathbb{E}[\deg_{-}^{V_{\ell,c}}(v_i)/|V_{\ell,c}| \bigg)|V_{\ell,c}|
\\ 
& 
=
D \bigg(\mathbb{E}[\deg_{-}^{V_{\ell,c}}(v_i)]/0.25n+8\cdot \sqrt{k}\sqrt{\log n}/0.25n \bigg|\bigg|\mathbb{E}[\deg_{-}^{V_{\ell,c}}(v_i)/0.25n \bigg)0.25n
\le 
\frac{64k\log n}{2\mathbb{E}[\deg_{-}^{V_{\ell,c}}(v_i)]}=
\\
&
\le 
\frac{64k\log n}{k \mathbb{P}[X[v_i],(\ell,c)]}
\le 16\log n \quad \quad \text{$\left[\sum_{\ell,c}\mathbb{P}[X[v_i],(\ell,c)] \le 1 \right]$}
\end{align*}

That is, 
\[
\Pr(\deg_{-}^{V_{\ell,c}}(v_i)> k \mathbb{P}[X[v_i],(\ell,c)]+ 8\sqrt{k}\sqrt{\log n}) \le n^{-16}
\]

Applying it for all $(\ell,c)$ and summing up we get
\[
\Pr(\deg_{-}(v_i) \ge k \cdot \sum_{(\ell,c)} \mathbb{P}[X[i],(\ell,c)] + 32\sqrt{k}\log n) \le n^{-15}
\]

Here we note that $\sum_{(\ell,c)} \mathbb{P}[X[i],(\ell,c)]=1$ and that $\sqrt{k}\sqrt{\log n}=o(k)$ as $k=\omega(\log n)$. This completes the upper bound. The lower bound follows similarly.
\end{proof}

Then, the distribution of the in-degree of the vertices can be obtained as follows.

\begin{lemma}
\label{lemma:expected-degree}
Let $G$ be a graph generated from $\sf{BM}(\{V_{\ell,c}\}_{(\ell,c) \in \{0,1\}^2},\mathbb{P},k)$. Then, 
for any vertex $v_i \in  V$, the expected in-degree of $v_i$ is
\[
\mathbb{E} \left[ \deg_{+}(v_i) \right]=k \cdot \sum_{(\ell',c')} \mathbb{P}[(\ell',c'),X[i]] 
\]
\end{lemma}
\begin{proof}
Let the vertices of $V$ be denoted as $v_1, \ldots, v_n$. Let $e_{i,j}$ denote the indicator random variable indicating which is 1 if there is an edge from $v_i$ to $v_j$, and $0$ otherwise. 

First, note that all $e_{i,j}$ are independent. Then, $\Pr(e_{i,j}=1)=\frac{\mathbb{P}[X[i],X[j]] \cdot k}{|V_{X[j]}|}$. Then, summing on the expectation we get, for any $v_i$
\begin{align*}
\mathbb{E}[deg_{+}(v_i)]= 
&\sum_{j \in [n]} \mathbb{E}[e_{j,i}]
=\sum_{(\ell',c') \in \{0,1\}^2}
|V_{\ell',c'}| \cdot \frac{\mathbb{P}[(\ell',c'),X[j]] \cdot k}{|V_{X[i]}|}
\\ = &
\sum_{(\ell',c')} \frac{n}{4} \cdot \frac{\mathbb{P}[(\ell',c'),X[j]] \cdot k}{n/4}
=   k \cdot \sum_{(\ell',c')} \mathbb{P}[(\ell',c'),X[i]] 
\end{align*}

\end{proof}

Then, we have the following bound on the indegree of the vertices.

\begin{lemma}   
\label{lem:deg-concentration}
Let $G$ be a graph generated from $\sf{BM}(\{V_{i,j}\}_{(i,j) \in \{0,1\}^2},\mathbb{P},k)$. Then, the in-degree of any vertex $v$ is bounded as 
\[
|\Pr(|\deg_{+}(v)-\mathbb{E}[\deg_{+}(v)]
| \ge 8\cdot \sqrt{k}\sqrt{\log n} ) \le n^{-16}
\]
\end{lemma}
\begin{proof}
We obtain this using Theorem~\ref{thm: chernoff}. 

For the upper tail, we have
\begin{align*}
\Pr(\deg_{+}(v)> \mathbb{E}[\deg_{+}(v)]+ 8\cdot \sqrt{k}\sqrt{\log n}) 
\le e^{-D(\mathbb{E}[\deg_{+}(v)]/n+ 8\cdot \sqrt{k}\sqrt{\log n}/n||\mathbb{E}[\deg_{+}(v)]/n)n}
\end{align*}

We note that the KL divergence between Bernoulli random variables $x,y$
$D(x||y)=x \ln(x/y) + (1-x)\ln((1-x)/(1-y))$, which is upper bounded by $(x-y)^2/2x$ when $x>y$. Then, 
$D(\mathbb{E}[\deg_{+}(v)]/n+4\cdot \sqrt{k}\sqrt{\log n}||\mathbb{E}[\deg_{+}(v)]/n)n$ can analyzed as
\begin{align*}
&
D(\mathbb{E}[\deg_{+}(v)]/n+ 8\cdot \sqrt{k}\sqrt{\log n}||\mathbb{E}[\deg_{+}(v)]/n)n
\le 
\frac{64k\log n}{2\mathbb{E}[\deg_{+}(v)]}=
\\
&
\le 
\frac{64k\log n}{k \cdot \sum_{\ell,c}\mathbb{P}[(\ell,c),X[v]]}
\le 16\log n & \text{$\left[\sum_{\ell,c}\mathbb{P}[(\ell,c),X[v] \le 4\right]$}
\end{align*}
Then, substituting, we get an upper bound of $n^{-16}$.
\end{proof}

This, along with the fact that the sum of entries of any column of $\mathbb{P}$ is $\Omega(1)$, gives the following fact. 
\begin{fact}
\label{fact: deg-core}
Let $G$ be a graph sampled from $\sf{BM}(\{V_{\ell,c}\}_{(\ell,c) \in \{0,1\}^2}$ where 
$k=\omega(\log n)$ and $\mathbb{P}[(\ell,c),(\ell,1)]=\Omega(1)$. Then for all vertices 
$v_i \in V_{\ell,1}$, we have $\deg_{+}(v)=\Omega(k)$ with probability $1-n^{-7}$.
\end{fact}
\begin{proof}
    This is straight from the fact that $\mathbb{E}[\deg_{+}(v_i)]=\Omega(k)$ and the tail deviation is $\sqrt{k}{\sqrt{\log n}}=o(k)$ with high probability.
\end{proof}

Then we make a connection between the concentration of any of the core/periphery blocks and degree of each vertices in the core.

\begin{lemma}
\label{lem:deg-concentration-relation}
Let $G$ be a graph sampled from  $\sf{BM}(\{V_{\ell,c}\}_{(\ell,c) \in \{0,1\}^2},\mathbb{P},k)$ where $k=\omega(\log n)$. 
Then, for any vertex $v_i \in V_{\ell,c}$ we have with probability $1-n^{-4}$,
\[
\deg_{+}(v_i)= k(1 \pm o(1)) +
k\cdot \sf{CC}_G(V_{\ell,c}) (1 \pm o(1)).
\]
\end{lemma}
\begin{proof}

Let us first recall the definition of concentration. We have $\sf{CC}_G(S)=(E(\bar{S},S)-E(S,\bar{S}))/|E(S,V)|$.

First, we note that $E(S,V)=k(1 \pm o(1))\cdot |V_{\ell,1}|$. Also, let $\deg_{(+,S)}(u)$ denote the number of edges connected to $u$ from the set $S$. 

Then, we have 
\begin{align*}
&\sf{CC}_G(V_{\ell,c})=
\frac{\sum_{v_i \in V_{\ell,c}} deg_{+}(v_i)
- \sum_{v_i \in V_{\ell,c}} deg_{(+,V_{\ell,c})}(v_i)
-\sum_{v_i \in \overline{V_{\ell,c}}} deg_{(+,V_{\ell,c})}(v_i)}{k(1 \pm o(1))\cdot|V_{\ell,c}|}
\\
&= \frac{\sum_{v_i \in V_{\ell,c}} deg_{+}(v_i)
- \sum_{v_i \in V} deg_{(+,V_{\ell,c})}(v_i)}{k(1 \pm o(1))\cdot|V_{\ell,c}|}
\\
&=
\frac{\sum_{v_i \in V_{\ell,c}} deg_{+}(v_i)
- k(1 \pm o(1)) \cdot |V_{\ell,c}|}{k(1 \pm o(1))\cdot|V_{\ell,c}|}
\end{align*}

That is,
\begin{equation}
\label{eq:deg-concentration-1}
\sf{CC}_G(V_{\ell,c})=\frac{1}{|k(1 \pm o(1)) \cdot V_{\ell,c}|}\cdot\sum_{v_i \in V_{\ell,c}}\deg_{+}(v)-1
\end{equation}
    
Then, we can use the fact that the in-degree of each vertex in $V_{\ell,c}$ is bounded tightly from Lemma~\ref{lem:deg-concentration}. Applying an union bound we get that with probability $1-n^{-6}$, 
$|\deg_{+}(v_i)-\deg_{+}(v_{i'})| \le 8\sqrt{k}\sqrt{\log n}$ for any $v_i,v_{i'} \in V_{\ell,c}$. 

Then, with the same probability, for any $v_i \in V_{\ell,c}$, we have that 
$|V_{\ell,c}\cdot \deg_{+}(v_{i})-\sum_{v_{i'} \in V_{\ell,c}} deg_{+}(v_i')|\le 32\sqrt{k}\sqrt{\log n}|V_{\ell,c}|$.

Furthermore note that as $k=\omega(\log n)$, we have $\sqrt{k}\sqrt{\log n}/k=o(1)$.
This implies with probability $1-n^{-4}$, 
for any $v_i \in V_{\ell,c}$,
\begin{align*}
&
\left|\sf{CC}_G(V_{\ell,c})-\left(\frac{|V_{\ell,c}|}{|k(1 \pm o(1)) \cdot V_{\ell,c}|}\cdot \deg_{+}(v_i)-1 \right ) \right|  \le 32\sqrt{k}\sqrt{\log n}/k
\\
\implies &
\left|\sf{CC}_G(V_{\ell,c})-\left(\frac{1}{k(1 \pm o(1))}\cdot\deg_{+}(v)-1 \right ) \right|  =o(1)
\\
\implies &
\frac{1}{k(1 \pm o(1))} \cdot \deg_{+}(v_i) =1+CC_G(V_{\ell,c}) \pm o(1)
\\
\implies &
\deg_{+}(v_i)= k+k\cdot \sf{CC}_G(V_{\ell,c}) \pm  o(k)\sf{CC}_G(V_{\ell,c}) \pm o(k)
\\
\implies &
\deg_{+}(v_i)= k(1 \pm o(1)) +
k\cdot \sf{CC}_G(V_{\ell,c}) (1 \pm o(1)).
\end{align*}

This completes the proof.
\end{proof}

\subsection{Proof of Theorem~\ref{thm: deg-limitaion}}

Now, we are ready to complete our first proof.
First, we know that $\deg(v_i)=\deg_{+}(v)+\deg_{-}(v)$,
where $\deg_{-}(v)=k \pm o(k)$. Then Lemma~\ref{lem:deg-concentration-relation} directly implies that with probability 
$1-n^{-4}$, 
$\deg(v_i) =2k+k\sf{CC_G}(V_{\ell,1}) \pm o(k)$.

Here, Fact~\ref{fact: deg-core} dictates  $\deg_{+}(v)=\Omega(k)$, and simply applying to Equation~\ref{eq: CC}, we get 
$\sf{CC}_G(V_{\ell,1})=\Omega(1)$. Then, we can write that
$deg(v_i)= k \cdot (2+\sf{CC}_G(V_{\ell,1})
\pm o(\sf{CC_G}(V_{\ell,1})))$, which completes the proof.

\subsection{Analysis of 1-step N-Rank: Algorithm~\ref{alg:NgRank}}

Here, recall that first, we obtain the score $F(v_i)$ for all vertices $v_i \in V$. It is easy to see, that $F(v_i)=\deg_{+}(v_i)$ for all vertices. 

Then in the next step, for each vertex $v_i$, we select $S_{v_i}$ as the set of vertices with a higher $F$ score. Then, we obtain
$\hat{F}(v_i)= \frac{F(v_i)}{\sf{average}_{v_j \in S_{v_i}} F(v_j)}$.

Then, we are ready to prove that $\hat{F}(v_i)$ is between $1-o(1)$ and $1$ for any core vertex $v_i$.

Recall that the graph has $(\Omega(1),o(1/k))-\sf{MCPC}$ structure. 
That is, $\sf{CC}_G(V_{\ell,1}) \ge \alpha+V_{\ell',0}
$ for any $(\ell,\ell')$ pair where $\alpha>0$ is a constant.

Let $CC$ be the minimum concentration among the cores. 
Then, Lemma~\ref{lem:deg-concentration-relation} dictates, that with probability $1-n^{-3}$,
\[
\text{If $v_i \in V_{\ell,0}$ then }\deg_{+}(v_i) \le k+k\cdot (CC-\alpha)+ o(k)
\]

On the other hand, for core vertices we have

\[
\text{If $v_i \in V_{\ell,1}$ then }\deg_{+}(v_i) \ge k+k\cdot CC- o(k)
\]

This implies that for any $v_i \in V_{\ell,1}$ and $v_j \in V_{\ell',0}$, $F_G(v_i)>F_G(v_j)$ with high probability.

Thus, in Algorithm~\ref{alg:NgRank} when we select $S_{v_i}$ for any $v_i \in V_{\ell,1}$, it does not include any periphery vertices. Then, we can show that the final score of all core vertices will be pretty similar.

\subsection{Proof of Theorem~\ref{thm: main}}

Consider any $V_{\ell,1}$. 
We first note that $\beta=o(1/k)$. Then it is easy to see that  $\mathbb{P}(V_{\ell,1},V_{\ell',1})= q=o(1/k)$ for any $\ell,\ell'$.

First we count the number of vertices in $V_{\ell,1}$ that has an outgoing edge to $V_{\ell',1}$ for some $\ell' \ne \ell$ (that is inter-core edges originating in $V_{\ell,1}$. This can be upper bounded as 
$\sum_{v_i \in V_{\ell,1}} \deg_{(+,V_{\ell',1})}(v_i)$. This sum has an expected value of $|V_{\ell,1}| \cdot \frac{|V_{\ell',1}| k \cdot o(1/k)}{|V_{\ell',1}|}=o(|V_{\ell,1}|)$. With high probability this can also be upper bounded by $o(|V_{\ell,1}|)+\sqrt{|V_{\ell,1}|}\sqrt{\log n}=o(|V_{\ell,1}|)$.

Let $S_{\ell,1}$ denote the vertices in $V_{\ell,1}$ that \emph{do not} have any outgoing edge to $V_{\ell',1}$. For any such vertex, $\max \{\hat{F}(v_j)\}_{v_j \in N_G(v_i)} \le k \cdot \sum_{(\ell',c')} \mathbb{P}[(\ell',c'),X[i]] +o(k)$.

On the other hand, $F_G(v_i) \ge k \cdot \sum_{(\ell',c')} \mathbb{P}[(\ell',c'),X[i]] -o(k)$.

Then we have
\begin{align*}
&
\hat{F}(v_i) \ge \frac{k \cdot \sum_{(\ell',c')} \mathbb{P}[(\ell',c'),X[i]] -o(k)}{k \cdot \sum_{(\ell',c')} \mathbb{P}[(\ell',c'),X[i]] +o(k)}
\\
&=
1- \frac{2\cdot o(k)}{k \cdot \sum_{(\ell',c')}\mathbb{P}[(\ell',c'),X[i]] +o(k)}
\\
&
\ge 1-o(1) &\text{ [As we know $k \cdot \sum_{(\ell',c')}\mathbb{P}[(\ell',c'),X[i]]=\Omega(k)$]}
\end{align*}

This completes our proof. 

Finally, we also show that the periphery vertices have a lower score than almost all core vertices.
\begin{lemma}[Separation between $\hat{F}$ score of cores and peripheries]
\label{lem: core-periphery-separation}
Let $G$ be a graph sampled from  $\sf{BM}(\{V_{\ell,c}\}_{(\ell,c) \in \{0,1\}^2},\mathbb{P},k)$ where $k=\omega(\log n)$. Then, for any $v_i \in V_{\ell,0}$ (a periphery vertex), $\hat{F}(v_i) < \min_{v_j \in V_{\ell,1}} \hat{F}(v_j)$.
\end{lemma}
\begin{proof}

Let $v_i \in V_{\ell,0}$ be a periphery vertex s.t. ${\sf CC}(V_{\ell,0})=CC$.   Then, its neighbors consist of some core vertices and other periphery vertices. 

Let $k_1$ be the number of neighbors of $v_i$ that is a core vertex. As $\mathbb{P}[(\ell,0),(\ell,1)]=\Omega(1)$, we know $k_1=\Omega(k)$. 

Furthermore, if $v_j \in V_{\ell,1}$, $\deg_{+}(v_j)$ (which is its $F(v_j)$ score) is lower bounded by $k+k\cdot (CC +\alpha) - o(k)$. On the other hand, 
$\deg_{+}(v_j) \le k+ k\cdot CC +o(k)$.

Then, 
$\sf{average}_{v_j \in S_{v_i}} F(v_j)$ is lower bounded by 
\begin{align*}
\sf{average}_{v_j \in S_{v_i}} F(v_j)
 \ge &
\frac{k_1\cdot \underset{v_j \in S_{v_j} \cap V_{\ell,1} }{\min} \deg_{+}(v_j) +(\deg_{-}(v_i)-k_1) \cdot \deg_{+}(v_i)}{\deg_{-}(v_i)}
\\
\ge & \deg_{+}(v_i) + \frac{k_1}{\deg_{v_i}} \cdot \left(\underset{v_j \in S_{v_j} \cap V_{\ell,1} }{\min} \deg_{+}(v_j)-\deg_{+}(v_i)  \right)
\\
\ge & 
\deg_{+}(v_i) + \frac{k_1}{\deg_{v_i}} \cdot
(k\cdot \alpha -o(k)) && \text{[Using the  lower bound on $\deg_{+}(v_j): v_j \in V_{\ell,1}$} \\ & && \text{ and upper bound on $\deg_{+}(v_i)$]}
\\
\ge & 
\deg_{+}(v_i) + 
C_3 \cdot k && \text{[For some constant $C_3$]}
\\
\ge & 
\deg_{+}(v_i) + 
C_4 \cdot \deg_{+}(v_i) && \text{[For some constant $C_4$ as $\deg_{+}(v_i)=\OO(k)$]}
\\ 
\ge & \deg_{+}(v_i)(1+C_4)
\end{align*}

Then, we have $\hat{F}(v_i) \le \frac{\deg_{+}(v_i)}{\deg_{+}(v_i)(1+C_4)} \le \frac{1}{1+C_4}$. That is, $\hat{F}(v_i)$ is upper bounded by a constant less than $1$. 
On the other hand, $\hat{F}(v_i)$ value of core vertices is upper bounded by $1-o(1)$. Then there exists a large enough $n_0$ such that for all $n \ge n_0$ there is a separation.
\end{proof}

\section{Meta Algorithm}
\label{app:meta-algorithm}

\begin{algorithm}[ht]
   \caption{A meta generalization:\\ 
   Meta-Relative-Rank $(t,p,q)$}
   \label{alg:metaRank}
\begin{algorithmic}
   \STATE {\bfseries Input:} Graph $G(V,E)$ and meta-parameters $t,p,q$.

    \FOR{i in 1:n}
    \STATE $F^{(t)}(v_i) \gets \sum_{j}A^t_{j,i}$    
    \ENDFOR

    \STATE

    \STATE  \COMMENT{1. \textcolor{teal}{Obtain a $p$-hop NeighborRank}}

    \STATE

    \FOR{$v_i \in V$}

    \STATE $S^{(p)}_{v_i} \gets
    \{ v_j:v_j \in N_{G,y}(v_i), 
    F^{(t)}(v_j) > F^{(t)}(v_i) \} \cup \{v_i\}$

    \STATE 
    
    \STATE $\hat{F}^{(t)}(v_i) \gets 
    {\underset{v_j \in S^{(p)}_{v_i}}{\avg}[F^{(t)}(v_j)}] \big/ {F^{(t)}(v_i)}$

    \ENDFOR

    \STATE

    \STATE \COMMENT{2. \textcolor{teal}
    {Recurse the process $q$ times}}

    \STATE
    
    \STATE $H \gets \hat{F}^{(t)}$, $counter \gets 0$
    
    \WHILE{ $counter<q$} 

    \FOR{$v_i \in V$}
    
    \STATE $S^{(p)}_{v_i} \gets
    \{ v_j:v_j \in N_{G,y}(v_i), 
    H(v_j) > H(v_i) \} \cup \{v_i\}$

    \STATE 
    
    \STATE $\hat{H}(v_i) \gets 
    {\underset{v_j \in S^{(p)}_{v_i}}{\avg}[H(v_j)}] \big/ {H(v_i)}$

    \ENDFOR

    \STATE $H \gets \hat{H}$, $counter \gets counter+1$.

    \ENDWHILE

     \RETURN $\hat{H}$

\end{algorithmic}
\end{algorithm}

We design the algorithm meta-algorithm by extending the idea of N-Rank (Algorithm~\ref{alg:NgRank}) in two ways (as we briefly discussed in Section~\ref{sec:graph-setting}.

1) There may be periphery vertices in the graph that have a high $F_G$ value compared to its $1$-hop neighborhood. To mitigate this issue, we can look at look at some $p$-hop neighborhood $N_{G,p}(v)$ of $v$ when selecting the reference set. 

As we look at a larger set of vertices for comparison, this method is less likely to report local maxima as core nodes and thus cause \emph{increased core prioritization}. On the other hand, if we look at a vertex from a sparse core, then a large-hop neighborhood may contain more vertices from other cores, and using them in the reference set reduces their final core. This leads to potentially a  \emph{lower balancedness.}

2)    
We have observed that our N-Rank approach
increases the balancedness in the initial centrality measure $F$. In this direction, we can recursively apply this process by first calculating the $p$-hop N-Rank value and then feeding it back to the algorithm as the initial centrality measure to further increase balancedness. We can apply this recursive process any $q \ge 1$ many times. 

The idea is that if the $p$-hop N-Rank has higher balancedness than the initial centrality measure, recursively applying the process should result in \emph{increased balancedness} up to a point. On the other hand, such a method can also amplify any loss of core prioritization due to $p$-hop N-Rank, and thus lead to potentially a   \emph{lower core prioritization}. 

\subsection{The choice of t for the initial centrality measure}
In the first step of our algorithm, we obtain an initial centrality measure $F^{(t)}$ by obtaining the $t$-th power of the graph's adjacency matrix and taking the sum of the columns. When $t=1$, this converges to the in-degree of a vertex (up to some multiplicative factor). When $t$ is larger, this can be thought of as a truncated variant of PageRank without any damping for $k$-regular graphs.

In our experiments, we set $t=1$ for the graphs generated by the $\sf{MCPC}$ block model and $t=\log |V|$ for both concentric GMM as well as real-world experiments. This is based on the intuition that if $t$ is large, then it can help discard local minima. In real-world data, there may be some periphery vertices that do not have any outgoing edges to a core and, as such, can have the highest $F^{(t)}$ value among its outgoing neighbors, thus obtaining a score of $1$. A larger $t$ can solve this issue.

On the other hand, if $t$ is too large, it can further increase the $F^{(t)}$ value of a core with higher concentration (as more random walks will be trapped there) and can reduce the balancing effect of the subsequent steps. While we do not need to do a hyperparameter tuning for our real-world experiments (fixing $t=\log |V|$ generates encouraging results), it is important to see whether a data-dependent value of $t$ can be obtained. 

In general, observing the impact of using different initial centrality measures on the overall performance is an important direction.


\section{More exploration of concentric GMM}
\label{app: conc-GMM}

In Section~\ref{sec: C-GMM}, we observed that in the concentric GMM setting, if the variances of one core-periphery pair are larger than the other pair, traditional centrality measures generate a biased ranking on the $K$-NN embedding of such a graph with two examples. Here, we do a larger-scale simulation to concretize this observation and further look at the ICEF-improvement vs. balancedness tradeoff of our methods. We use the same setup as in Section~\ref{app: conc-GMM}, but run our methods on graphs generated with several different parameters. 

We set $d=20$, size of each core or periphery $V_{\ell,c}=2000$, and the two centers being ${\bf c_1}=\{0\}^d$ and ${\bf c_2}=0.3\cdot \{1\}^d$. Then, we set the following variances of the distributions, parameterized by $\gamma$. 

$\sigma_{0,1}=0.1, \sigma_{1,1}=0.3$ and $\sigma_{1,c}=\gamma \cdot \sigma_{0,c}$ for some $\gamma \ge 1$.
Then, we generate $V_{\ell,c}$ many points with a $d$-dimensional isotropic Gaussian with $\bf c_{\ell}$ as the center and $\sigma_{\ell,c}$ as variance, with the cores being generated with lower variance distribution than their corresponding periphery.

Then, when $\gamma=1$, both the cores  (and the peripheries) are generated from 
distributions of the same variance $0.1$ (0.3 for the peripheries). As a result, the cores have similar concentrations, and the graph satisfies a $(0.2,0.75)$-$\sf{MCPC}$ structure. Then, the average ICEF of the induced subgraph, as well as the total balancedness of the ranking of all methods, are very similar. 

Next, we increase $\gamma$ slowly, and it results in the variance of the second community being higher both for the core and the periphery. The results are captured in Figure~\ref{fig: conc-GMM-total}. While the average ICEF of all the methods is pretty similar (within $4\%$), the traditional centrality measures have significantly higher balancedness, becoming as high as $20\%$ for $\gamma=2$. 

Furthermore, among our methods, RN-Rank seems to have the highest balancedness, whereas N-Rank seems to have the highest ICEF. This tradeoff indicates the possibility of more interesting algorithms via the relative centrality framework going forward.

        \begin{figure}[t]
          \subfigure[\label{fig-GMM-ICEF} Average ICEF]{\includegraphics[scale=0.45]{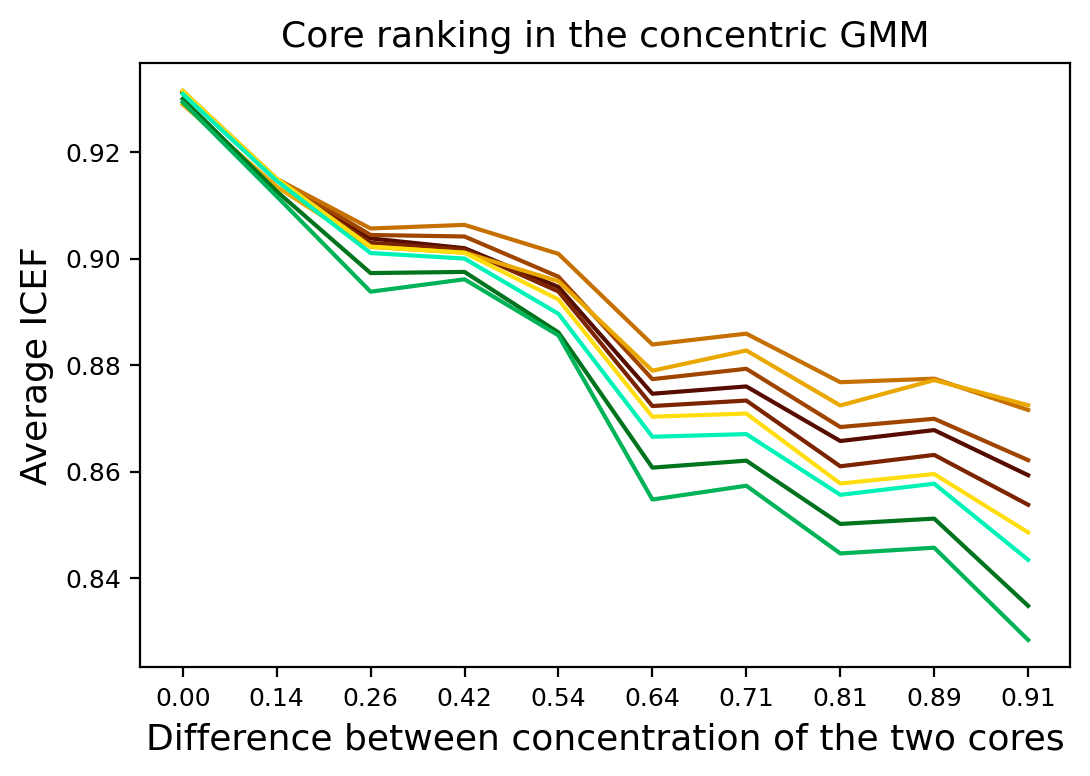}}
          \subfigure[\label{fig-GMM-balanced} Balancedness]{\includegraphics[scale=0.45]{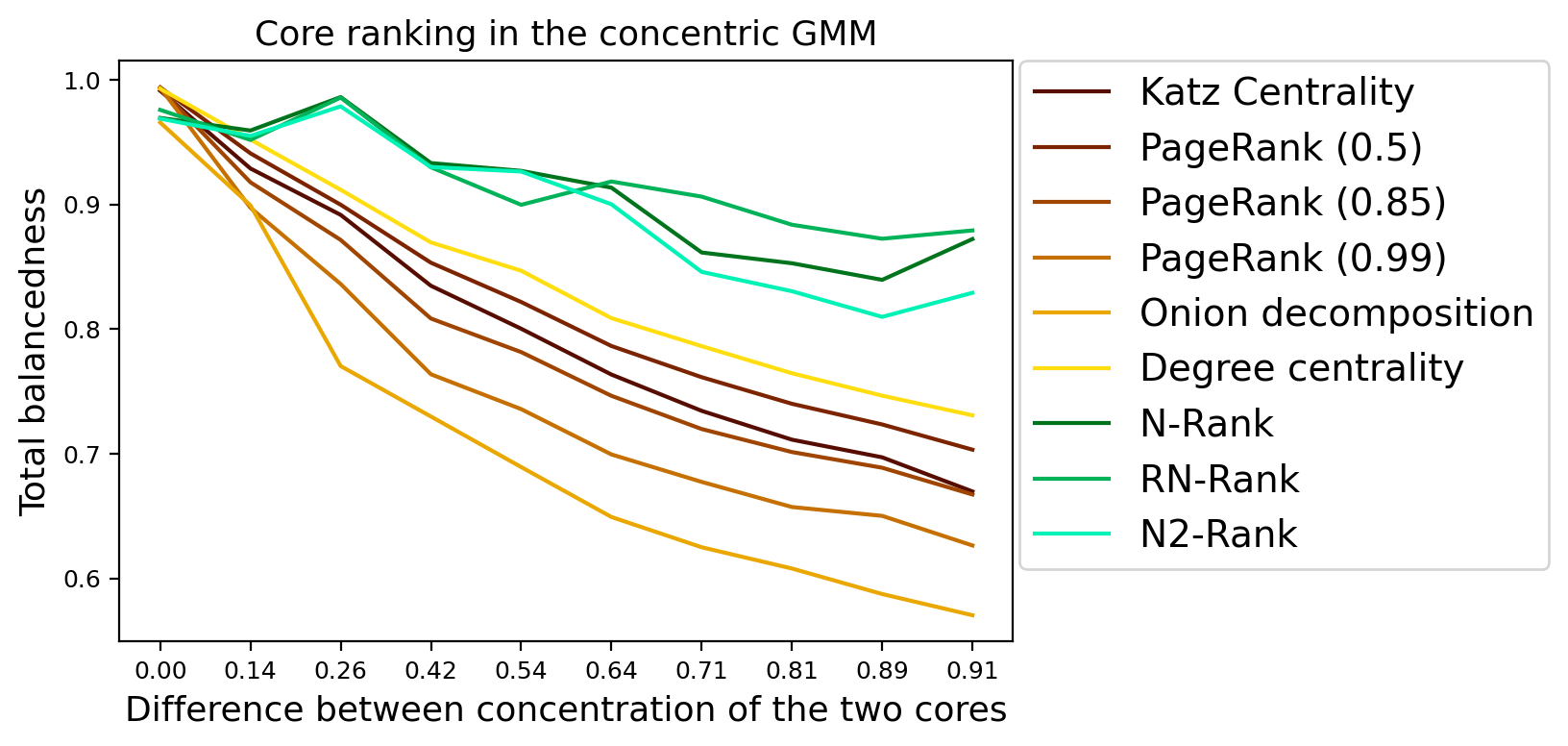}}
          \caption{Effect of core ranking on the 20-NN embedding of data generated with the concentric GMM for different instantiations of the model}
          \label{fig: conc-GMM-total}
        \end{figure}

\section{Single-cell data}
\label{app:single-cell}

\paragraph{Datasets}
First, we provide a detailed description of the datasets we use in Table~\ref{tab:dataset-details}.
Note that each dataset has annotated labels that we use to verify the performance of our algorithm.  Next, we describe the experimental setup. For each dataset, we first log-normalize it and then apply PCA dimensionality reduction of dimension 50, which is a standard pipeline in the single-cell analysis literature~\cite{single-cell-duo}. Then, we obtain $20$-NN embedding of each of the datasets and obtain a ranking of the vertices via both the baseline and traditional centrality measures as well as our relative centrality approach. Then, we select some $c$-fraction of the vertices from the top of the ranking and calculate the intra-community edge fraction of the induced subgraph by these vertices and also the preservation ratio, using the annotated labels.

\begin{table}
    \centering
    \begin{tabular}{ccccc}
    \toprule
       Name & Abbreviation & \# of points & \# of  communities & Source\\ \midrule
       Baron\_Human & BM & 8569 & 14 & \cite{single-cell-7data}\\
       Baron\_Mouse & BH  & 1886 & 13 & \cite{single-cell-7data} \\
       Muraro & Mu  & 2122 & 9 & \cite{single-cell-7data} \\
       Segerstolpe & Se  & 2133 &  13 & \cite{single-cell-7data}\\
       Xin & Xi & 1449 & 4 & \cite{single-cell-7data}\\
       Zhengmix8eq & Zh  & 3994 & 8 & \cite{single-cell-duo} \\
       T-cell dataset & Tcell  & 5759 & 10 & \cite{nature-medicine-t-cell} \\
       ALM & ALM  & 10068 & 136 &  \cite{single-cell-ALM-VISP}\\
       AMB & AMB  & 12382 & 110 &  \cite{single-cell-7data}\\
       TM & TM  & 54865 & 55 & \cite{single-cell-7data}\\
       VISP & VISP  & 15413 & 135 & \cite{single-cell-ALM-VISP} \\
        \bottomrule
    \end{tabular}
    \caption{Details of the scRNA datasets we use}
    \label{tab:dataset-details}
\end{table}

\paragraph{ICEF and preservation ratio of all datasets}
We put the change in the intra-community edge fraction as well as the preservation ratio by changing the fraction of points we select in the ranking for all of the $11$ single-cell datasets that we consider in this paper in Appendix~\ref{app:scRNA-image}. As can be observed from the figures, our methods have comparable improvement in intra-community edge fraction to the baselines. However, our methods generally have a superior preservation ratio. We make the following three observations. 

i) Our RN-Rank method provides the best preservation ratio among all the methods, and it has slightly lower ICEF improvement. Here, note that for different values of $c$, the baseline methods, in fact, completely fail to include many more underlying communities than our relative centrality based methods.

ii) Among our methods, N2-Rank provides the highest improvement in ICEF and has a lower preservation ratio than our other methods. In this direction, a better understanding of the preservation ratio-ICEF improvement tradeoffs of our framework is a very interesting future direction.

iii) Finally, for the Zhengmix8eq dataset, the traditional centrality measures do not provide any improvement in ICEF via subset selection, as can be observed in Figure~\ref{fig:scRNA-6}. This further points to weaknesses in the traditional centrality measures and the power of our relative centrality framework.

\paragraph{Balancedness of ranking}
We note that the balancedness values for the datasets Xin, Zheng, Tcell, and ALM are moderate, and we observe the same patterns as with the preservation ratio, with RN-Rank obtaining the best results. The other datasets show that the balancedness AUC values are very small, less than $0.1$. This indicates that at least one cluster is lost for these datasets when we filter out points. We attribute this to these datasets having several very small clusters. Improving our algorithms to have non-negligible (worst case) balancedness in such datasets is an important future direction. 

\paragraph{Purity improvement upon core-ranking based point selection}

Here, we put the ICEF, purity, and preservation ratio of the top $20\%$ points for different CR algorithms for all the datasets.

\begin{table}[ht]
    \centering
    \hspace*{-.7in}
    \resizebox{1.2\linewidth}{!}{
    \begin{tabular}{cccccccccccc}
      Datasets  & BM &  BH & MU & Se & Xi & Zh & Tcell  & ALM  & AMB &TM &VISP \\
      \# of points & 8569 & 1886 & 2122 & 2133 & 1449 & 3994 & 5759  & 10068  & 12832 &54865 &15413 \\ 
    Metrics & PR ~NMI& PR ~NMI& PR ~NMI& PR ~NMI& PR ~NMI& PR ~NMI&PR ~NMI&PR ~NMI&PR ~NMI&PR ~NMI&PR ~NMI\\
    \toprule
     \makecell{Original \\ values}   &  1.00  0.75 & 1.00 0.70  & 1.00 0.74 & 1.00 0.67  &  1.00  0.60 & 1.00 0.72  &  1.00 0.46  &  1.0 0.74   &  1.00 0.74 & 1.00 0.82 & 1.00 0.69\\
    \midrule

    \textcolor{darkolivegreen}{RN-Rank}  &     0.83 0.70  &  0.68 0.75  &  0.84 0.78  &  0.56 0.72  &  0.72 0.57  &  0.85 0.76  &  0.89 0.60  &  0.61 0.84  &  0.69 0.83  &  0.87 0.80  &  0.61 0.81  \\  

    \textcolor{darkgreen}{N2-Rank} &   0.85 0.74  &  0.50 0.78  &  0.74 0.80  &  0.52 0.74  &  0.72 0.60  &  0.82 0.77  &  0.87 0.64  &  0.51 0.85  &  0.63 0.86  &  0.85 0.82  &  0.56 0.83  \\
    \midrule
    Katz &     0.80 0.77  &  0.48 0.75  &  0.73 0.78  &  0.50 0.72  &  0.61 0.58  &  0.90 0.80  &  0.60 0.61  &  0.43 0.86  &  0.51 0.87  &  0.74 0.82  &  0.49 0.85  \\  

PageRank (0.5) &   0.80 0.76  &  0.51 0.77  &  0.75 0.78  &  0.50 0.70  &  0.64 0.57  &  0.93 0.77  &  0.69 0.63  &  0.47 0.85  &  0.57 0.86  &  0.76 0.82  &  0.53 0.84  \\  

PageRank (0.85) &  0.79 0.77  &  0.49 0.75  &  0.74 0.79  &  0.51 0.72  &  0.55 0.57  &  0.88 0.77  &  0.60 0.62  &  0.42 0.86  &  0.49 0.88  &  0.73 0.83  &  0.47 0.85  \\  

PageRank (0.99) &    0.79 0.78  &  0.45 0.76  &  0.73 0.79  &  0.50 0.72  &  0.50 0.52  &  0.83 0.79  &  0.57 0.63  &  0.40 0.87  &  0.46 0.88  &  0.72 0.83  &  0.45 0.85  \\  

Onion &   0.71 0.78  &  0.34 0.75  &  0.51 0.69  &  0.22 0.53  &  0.40 0.45  &  0.92 0.74  &  0.50 0.44  &  0.24 0.82  &  0.35 0.88  &  0.45 0.79  &  0.35 0.83 \\  

Degree &   0.76 0.75  &  0.52 0.77  &  0.76 0.75  &  0.50 0.70  &  0.70 0.60  &  0.96 0.75  &  0.73 0.60  &  0.52 0.84  &  0.61 0.85  &  0.78 0.82  &  0.57 0.82  \\
    \bottomrule
    \end{tabular}
    }
    \caption{Preservation ratio and NMI of Louvain of top $20\%$ ranked points}
    \label{tab:single-cell-preserve-NMI-full}
\end{table}

\paragraph{NMI improvement upon core-ranking based point selection}

Then, in Table~\ref{tab:single-cell-preserve-NMI}, we observe the improvement in the NMI outcome of Louvain when applied on the top $20\%$ of the ranked points by the different methods, along with the preservation ratio of the selected subset. As with the purity, all core ranking methods give subsets that have similar improvements in the NMI. 

\begin{table}[H]
    \centering
    \hspace*{-.7in}
    \resizebox{1.2\linewidth}{!}{
    \begin{tabular}{cccccccccccc}
      Datasets  & BM &  BH & MU & Se & Xi & Zh & Tcell  & ALM  & AMB &TM &VISP \\
      \# of points & 8569 & 1886 & 2122 & 2133 & 1449 & 3994 & 5759  & 10068  & 12832 &54865 &15413 \\ 
    Metrics & PR ~NMI& PR ~NMI& PR ~NMI& PR ~NMI& PR ~NMI& PR ~NMI&PR ~NMI&PR ~NMI&PR ~NMI&PR ~NMI&PR ~NMI\\
    \toprule
     \makecell{Original \\ values}   &  1.00  0.75 & 1.00 0.70  & 1.00 0.74 & 1.00 0.67  &  1.00  0.60 & 1.00 0.72  &  1.00 0.46  &  1.0 0.74   &  1.00 0.74 & 1.00 0.82 & 1.00 0.69\\
    \midrule

    \textcolor{darkolivegreen}{RN-Rank}  &     0.83 0.70  &  0.68 0.75  &  0.84 0.78  &  0.56 0.72  &  0.72 0.57  &  0.85 0.76  &  0.89 0.60  &  0.61 0.84  &  0.69 0.83  &  0.87 0.80  &  0.61 0.81  \\  

    \textcolor{darkgreen}{N2-Rank} &   0.85 0.74  &  0.50 0.78  &  0.74 0.80  &  0.52 0.74  &  0.72 0.60  &  0.82 0.77  &  0.87 0.64  &  0.51 0.85  &  0.63 0.86  &  0.85 0.82  &  0.56 0.83  \\
    \midrule
    Katz &     0.80 0.77  &  0.48 0.75  &  0.73 0.78  &  0.50 0.72  &  0.61 0.58  &  0.90 0.80  &  0.60 0.61  &  0.43 0.86  &  0.51 0.87  &  0.74 0.82  &  0.49 0.85  \\  

PageRank (0.5) &   0.80 0.76  &  0.51 0.77  &  0.75 0.78  &  0.50 0.70  &  0.64 0.57  &  0.93 0.77  &  0.69 0.63  &  0.47 0.85  &  0.57 0.86  &  0.76 0.82  &  0.53 0.84  \\  

PageRank (0.85) &  0.79 0.77  &  0.49 0.75  &  0.74 0.79  &  0.51 0.72  &  0.55 0.57  &  0.88 0.77  &  0.60 0.62  &  0.42 0.86  &  0.49 0.88  &  0.73 0.83  &  0.47 0.85  \\  

PageRank (0.99) &    0.79 0.78  &  0.45 0.76  &  0.73 0.79  &  0.50 0.72  &  0.50 0.52  &  0.83 0.79  &  0.57 0.63  &  0.40 0.87  &  0.46 0.88  &  0.72 0.83  &  0.45 0.85  \\  

Onion &   0.71 0.78  &  0.34 0.75  &  0.51 0.69  &  0.22 0.53  &  0.40 0.45  &  0.92 0.74  &  0.50 0.44  &  0.24 0.82  &  0.35 0.88  &  0.45 0.79  &  0.35 0.83 \\  

Degree &   0.76 0.75  &  0.52 0.77  &  0.76 0.75  &  0.50 0.70  &  0.70 0.60  &  0.96 0.75  &  0.73 0.60  &  0.52 0.84  &  0.61 0.85  &  0.78 0.82  &  0.57 0.82  \\
    \bottomrule
    \end{tabular}
    }
    \caption{Preservation ratio and NMI of Louvain of top $20\%$ ranked points}
    \label{tab:single-cell-preserve-NMI}
\end{table}

\section{More comparisons with existing work}
\label{app:comp}

\subsection{Core detection algorithms in single-core periphery structure}

As we discussed, there exists a large literature of core-detection algorithms particularly focused on single core-periphery structure~\cite{core-periphery-survey,core-periphery-survey-revisited,CP-survey-new}. In this direction, the recent and comprehensive survey~\cite{CP-survey-new} cited that most of the core-periphery detection algorithms need $|E|^2\log |V|$ or more time and highlighted centrality measures as being efficient. This, along with the ease of applying these methods to a multi-core periphery structure, motivated our choice of baseline.

\subsection{Research on existence of multiple cores}
As we discussed, \cite{MCPC-directed} seems to be the only notable work in the literature that considered directed graphs with multiple cores. 
In this direction, we apply their core-detection algorithm to our $\sf{MCPC}$-block model. Their method requires the knowledge of the number of cores and the number of peripheries. When the graph consists of $2$ cores and $2$ peripheries (as in our block model), their method involves first obtaining a $4$-dimensional score using the popular HITS algorithm. Then, they apply $K$-Means with $K=4$ to the four-dimensional dataset to separate into $4$ blocks.

In this direction, we apply their method to a graph generated with the block probabilities in Table~\ref{tab:beta} with $\gamma=0$ and $n=4000$. Note that this is the simplest setting, where both cores have identical behavior in terms of core concentration, inter-core edges as well as overall community structure. We show the outcome in Figure~\ref{fig:HITS-Kmeans}. As we can observe, the output is not close to the ground truth. This can, in part, be attributed to the fact that their structure, although a multi-core directed one, is quite different from our $\sf{MCPC}$ structure.

\begin{figure}[t]
    \centering
\includegraphics[scale=0.4]{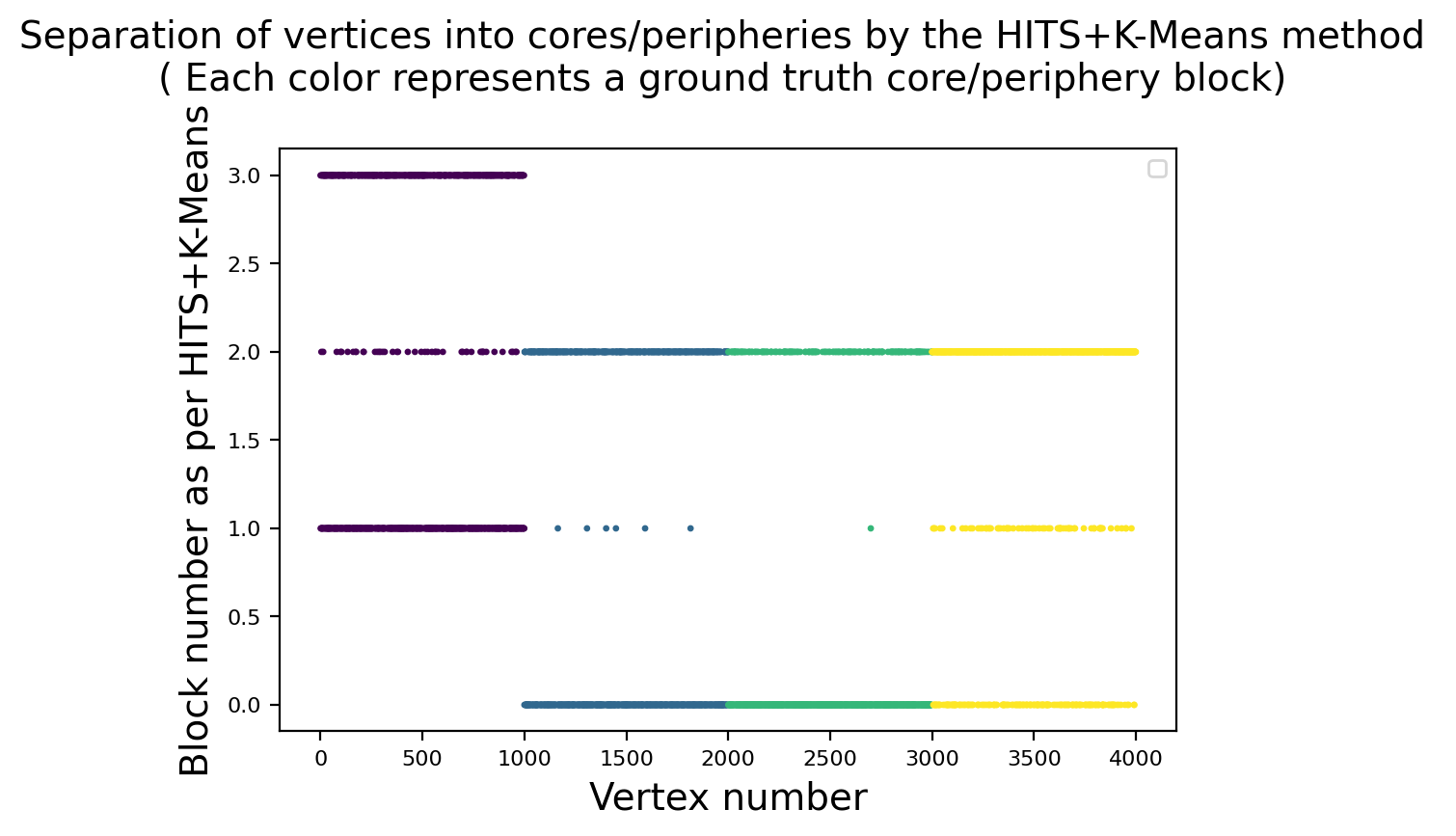}
    \caption{Performance of the HITS+K-Means algorithm in \cite{MCPC-directed} when applied to balanced $\sf{MCPC}$-block model}
    \label{fig:HITS-Kmeans}
\end{figure}

Finally, \cite{MCPC-hybrid-1} seems to be the closest structure to our $\sf{MCPC}$ structure, even though it is in an undirected graph setting. Indeed, this seems to be one of the initiating works in aiming to understand the coexistence of community and core-periphery structure in a systematic manner, and is thus an important contribution. Their approach is as follows. 

They also assume that the graph is partitioned into some ground truth communities, and each community has some core and some periphery vertices. They assume a hypothesis model in which the probability of an edge between $v_i$ and $v_j$ is 
$a\delta_{i,j}(C_i+C_j)+b$. We now decompress this definition. 
Here $\delta_{i,j}=1$ iff $i$ and $j$ belong to the same community, and $0$ otherwise. Next, $C_i=1$ if $v_i$ is a core vertex. That is, 
i) All inter-community edge probability is a fixed $b$,
ii) Intra-community core-periphery edge probability is $a+b$, and 
iii) Intra-core edge probability is $2a+b$.

Despite its expressibility, this structure has a few significant shortcomings.

First, it does not consider that the edge density can differ for different cores. We quantify this phenomenon with the concentration of a set of vertices and show that when cores have different concentrations, it can lead to many core ranking algorithms performing in an unbalanced manner, which we mitigate with our novel relative centrality framework.

Next, the model also does not consider that the inter-community edge probability between core vertices is less than between peripheral vertices. We capture this in our $(\alpha,\beta)-\sf{MCPC}$ structure definition, and this observation allows us to obtain subsets of real-world datasets with better community structure by using our core ranking algorithms.

Finally, they do not present any core detection algorithm beyond a maximum likelihood approach w.r.t the inference model we discussed. It is well known that such methods may have very slow convergence. Indeed, the experiments in \cite{MCPC-hybrid-1} consider graphs with less than $200$ vertices. In comparison, we apply our methods to datasets with $>50,000$ points and terminate generally in less than $10$ seconds, owing to our fast $\OO(\log |V||E|+k|V|)$ run time for N-Rank and RN-Rank and 
$\OO(\log |V||E|+k^2|V|)$ for N2-Rank.


\newpage
\section{Plots of ICEF of induced subgraph and preservation ratio of all single-cell datasets}
\label{app:scRNA-image}

\begin{figure}[H]
    \centering
\subfigure{\includegraphics[scale=0.45]{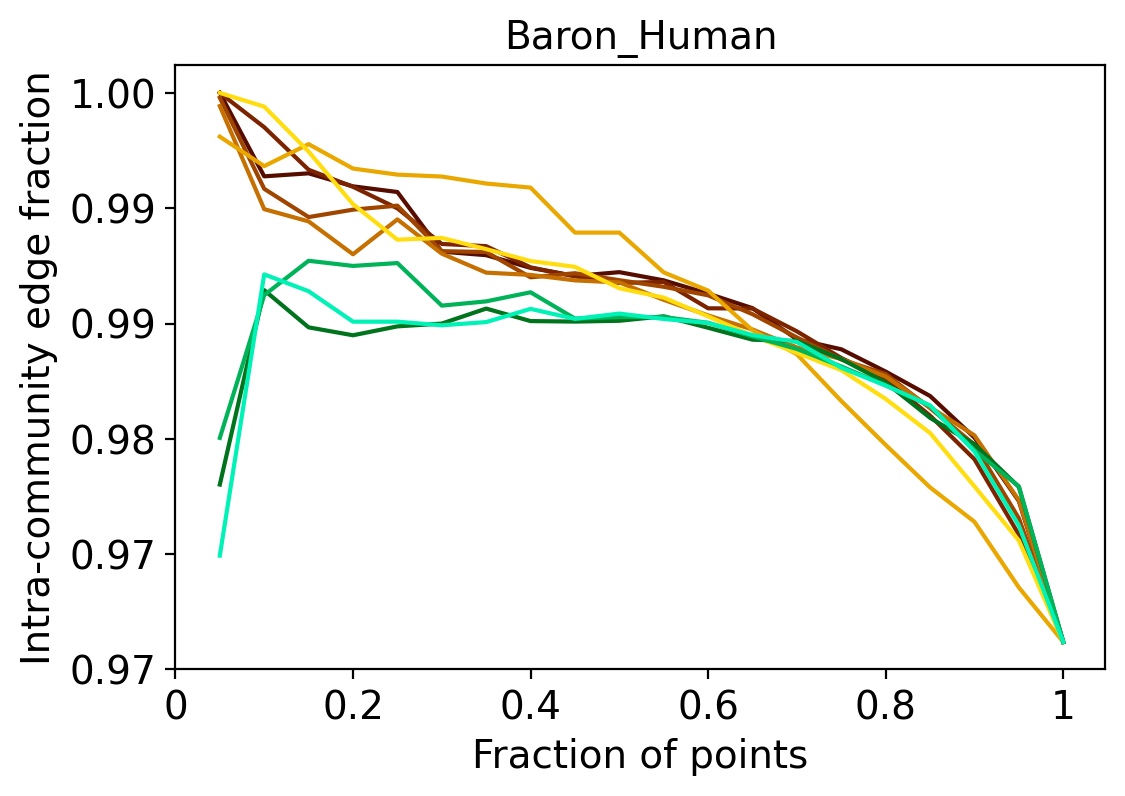}}
\subfigure{\includegraphics[scale=0.45]{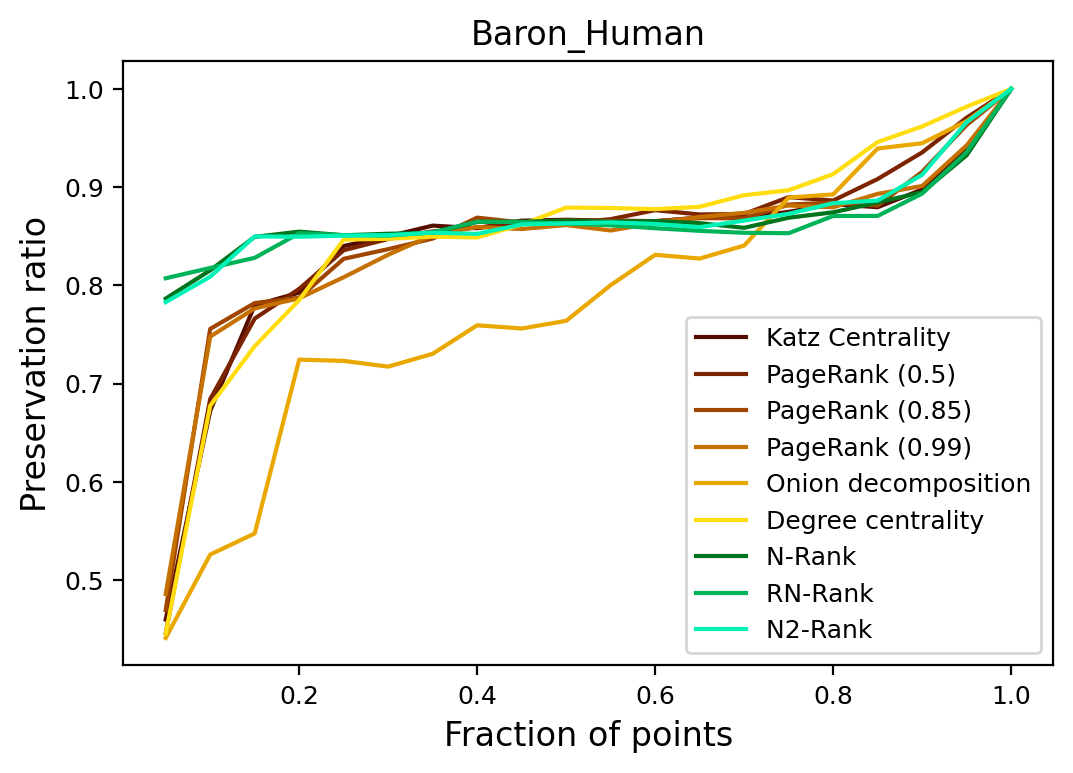}}
    \caption{Baron Human dataset}
    \label{fig:scRNA-1}
\end{figure}

\begin{figure}[H]
    \centering
\subfigure{\includegraphics[scale=0.45]{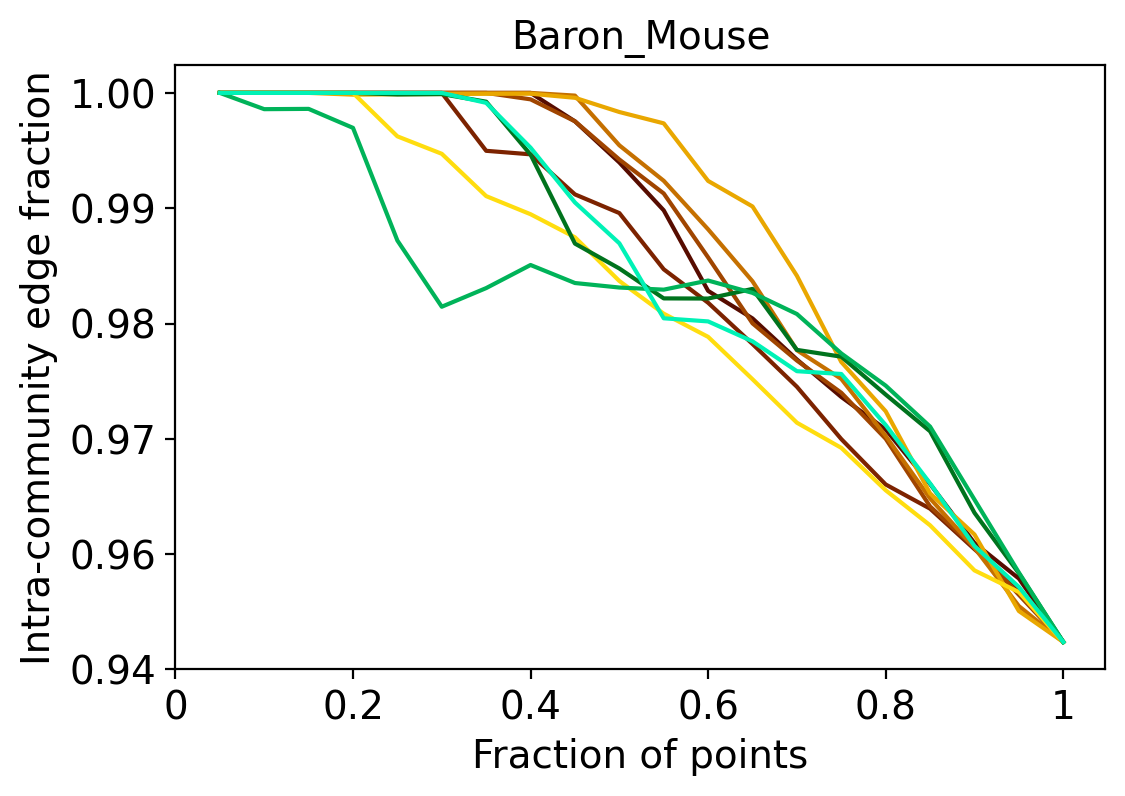}}
\subfigure{\includegraphics[scale=0.45]{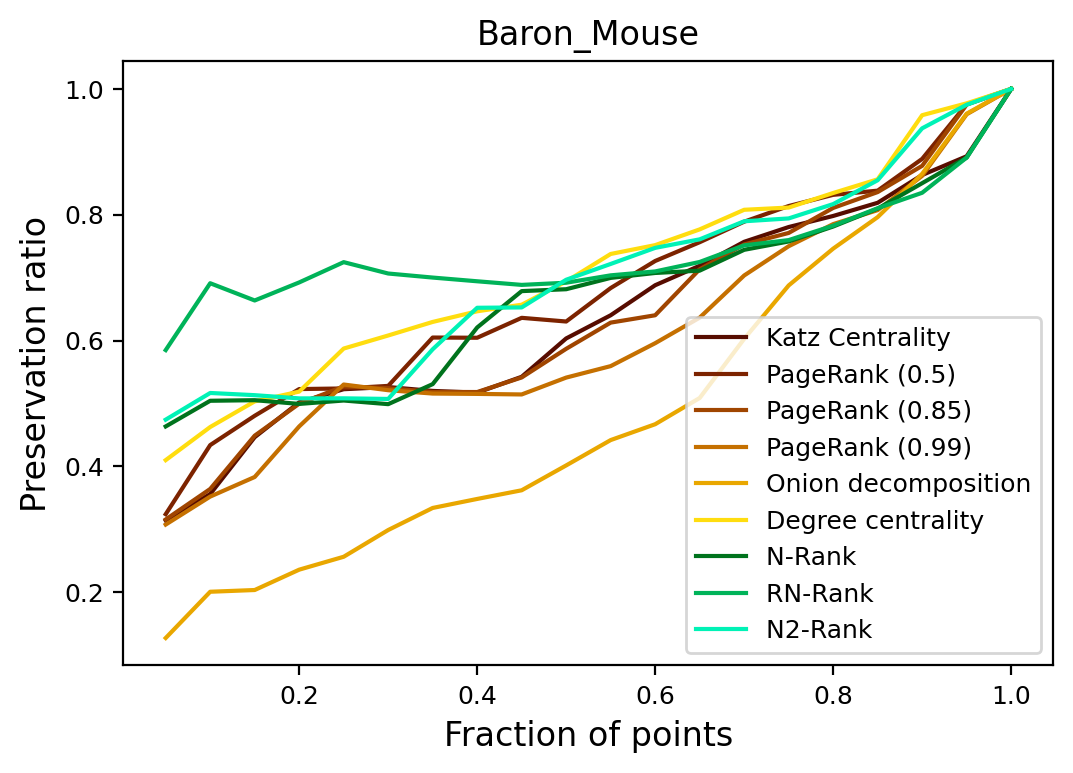}}
    \caption{Baron Mouse dataset}
    \label{fig:scRNA-2}
\end{figure}

\begin{figure}[H]
    \centering
\subfigure{\includegraphics[scale=0.45]{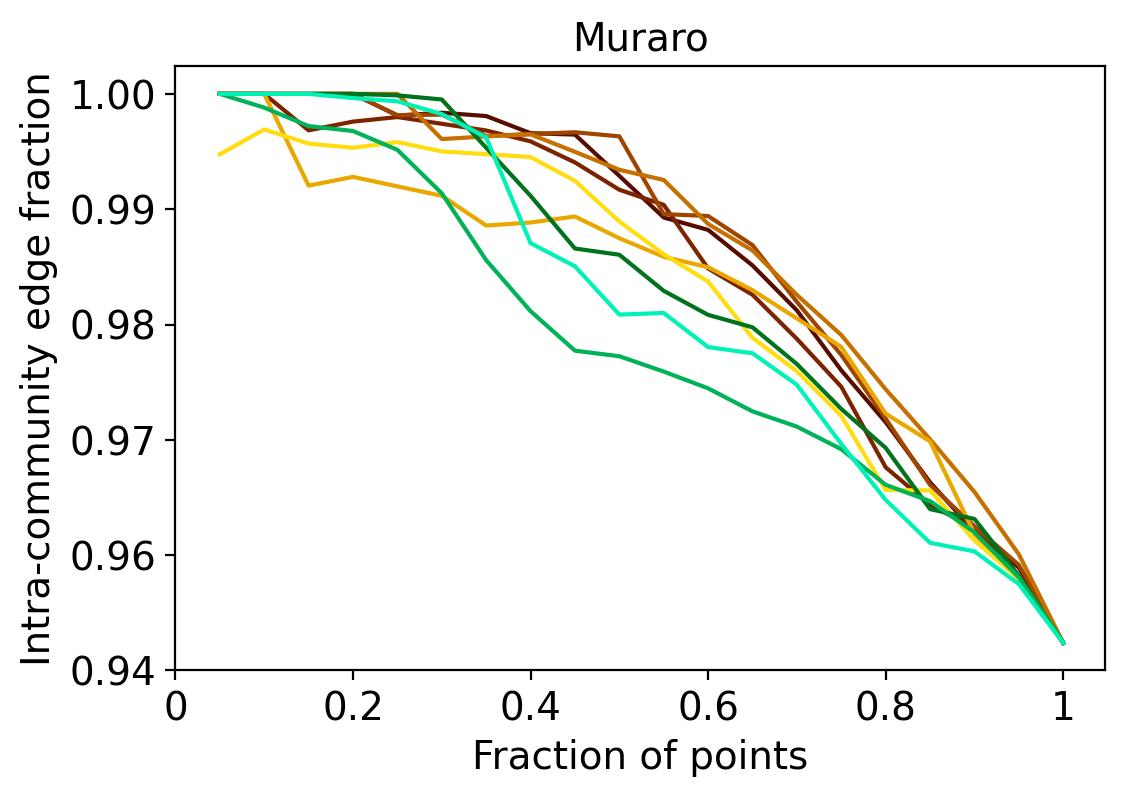}}
\subfigure{\includegraphics[scale=0.45]{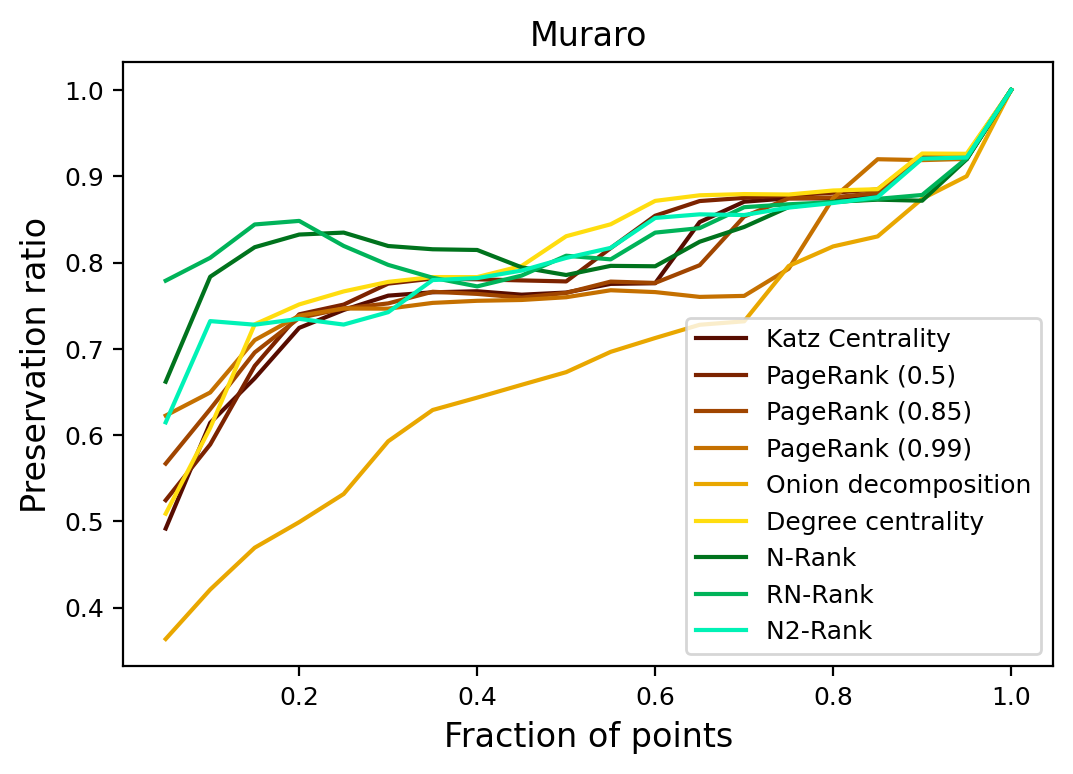}}
    \caption{Muraro dataset}
    \label{fig:scRNA-3}
\end{figure}

\begin{figure}[H]
    \centering
\subfigure{\includegraphics[scale=0.45]{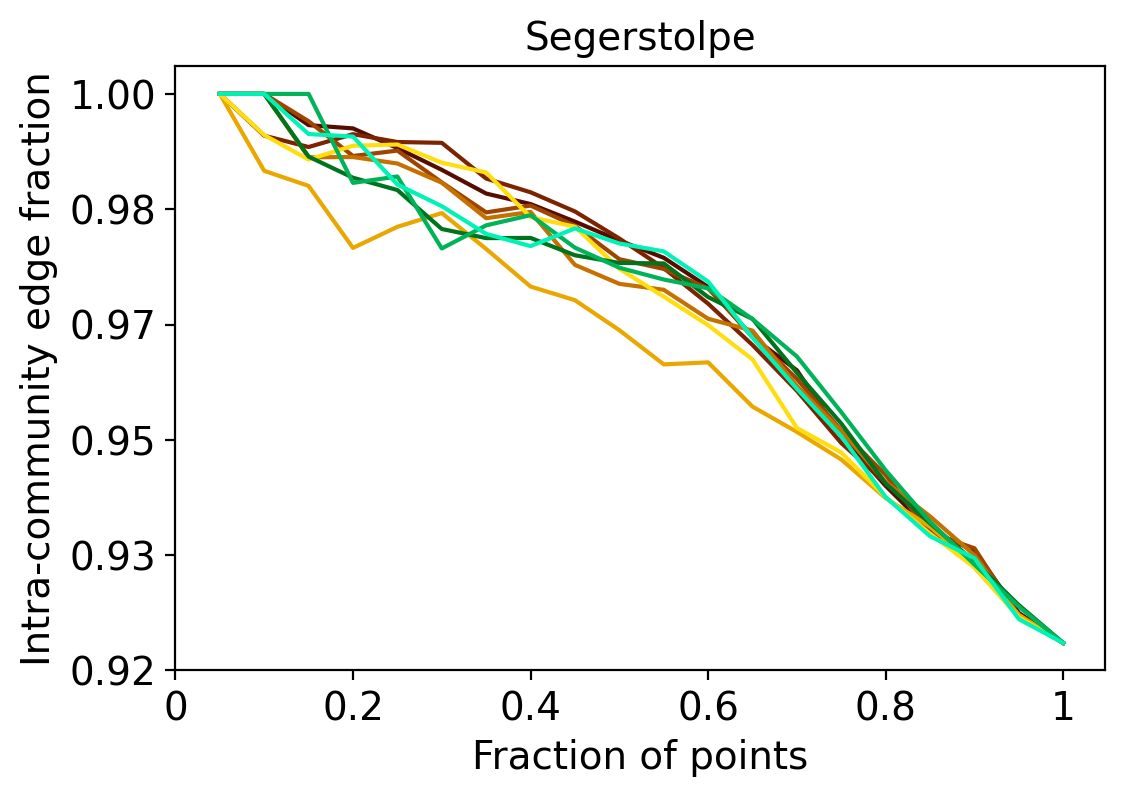}}
\subfigure{\includegraphics[scale=0.45]{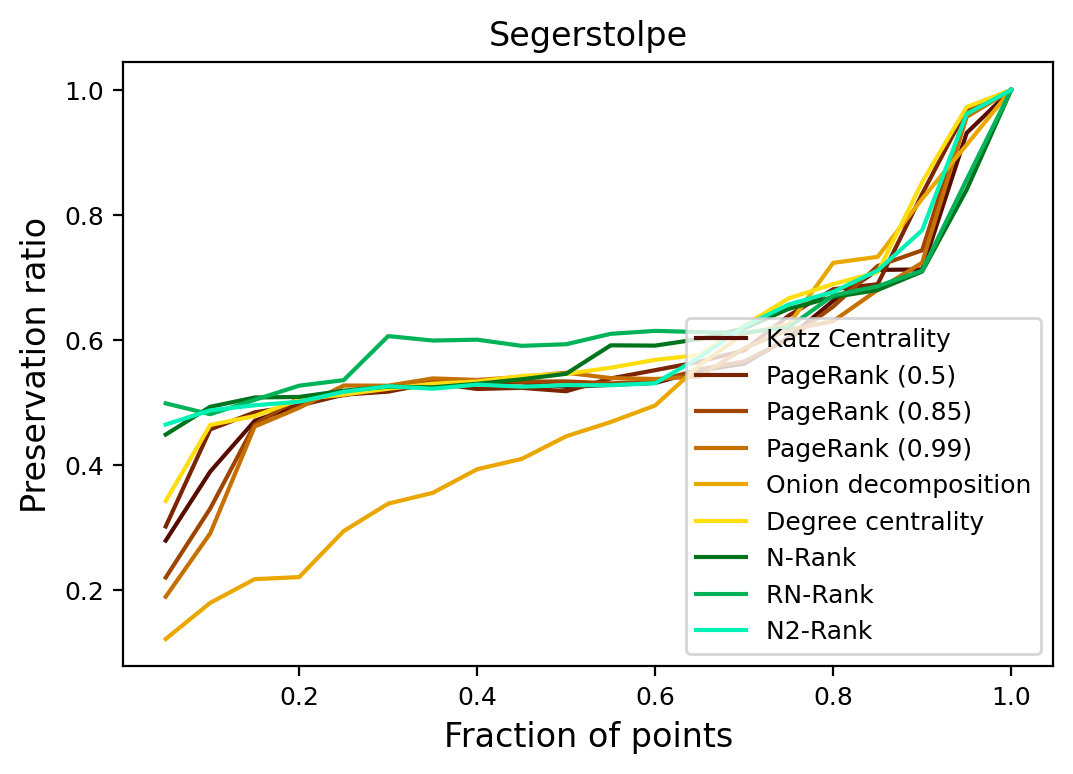}}
    \caption{Segerstolpe dataset}
    \label{fig:scRNA-4}
\end{figure}

\begin{figure}[H]
    \centering
\subfigure{\includegraphics[scale=0.45]{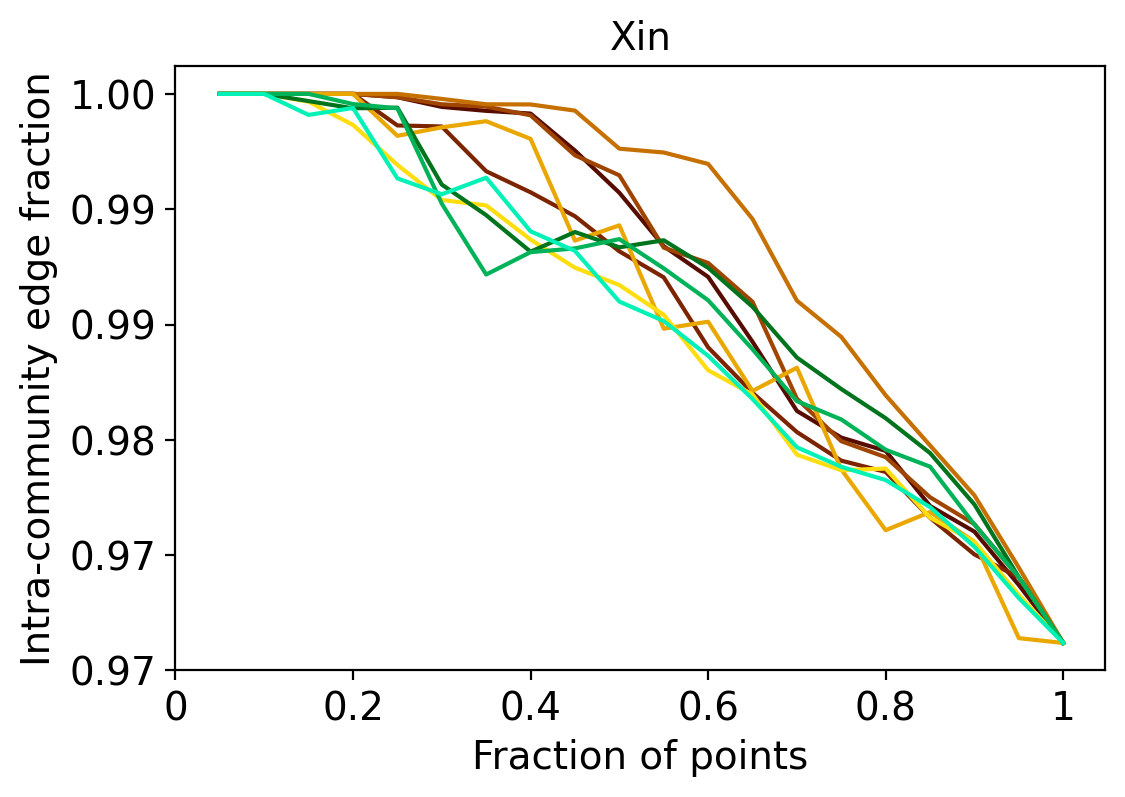}}
\subfigure{\includegraphics[scale=0.45]{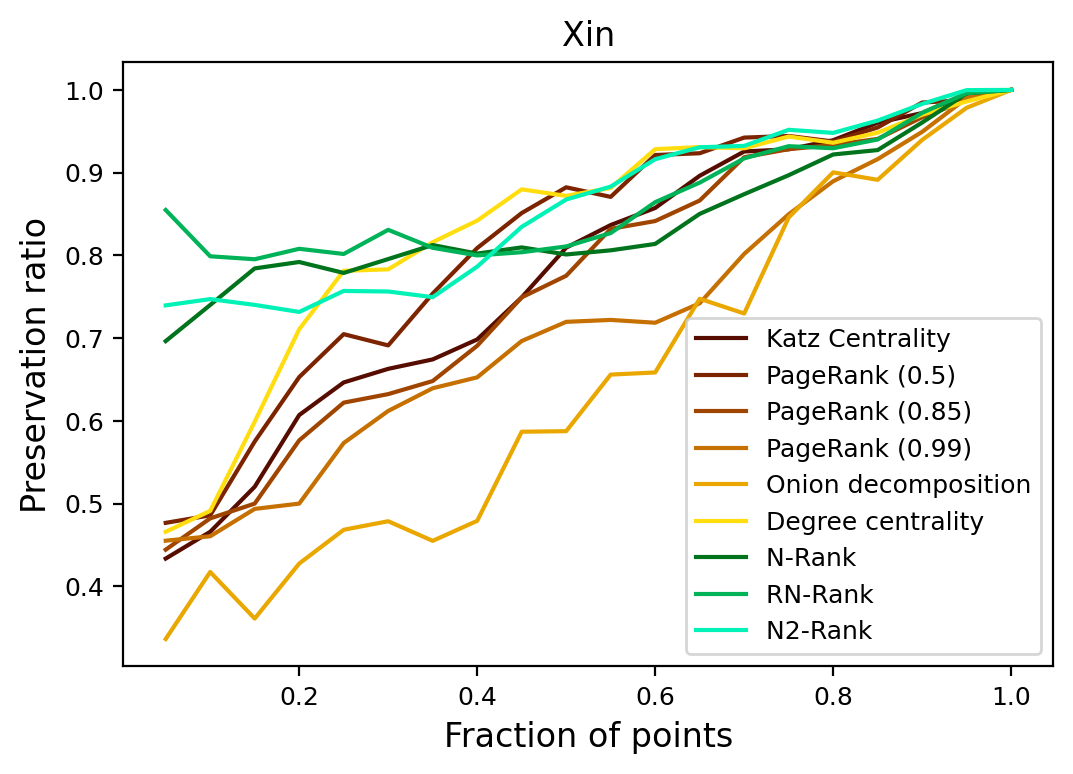}}
    \caption{Xin dataset}
    \label{fig:scRNA-5}
\end{figure}

\begin{figure}[H]
    \centering
\subfigure{\includegraphics[scale=0.45]{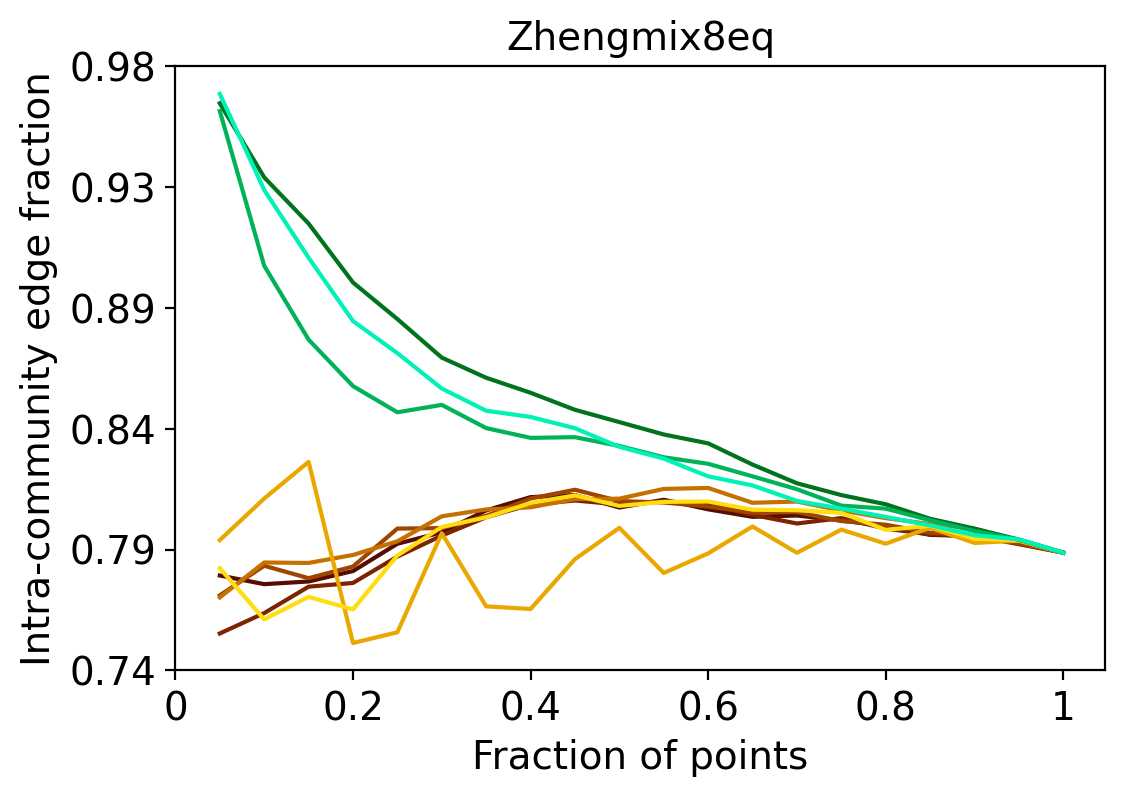}}
\subfigure{\includegraphics[scale=0.45]{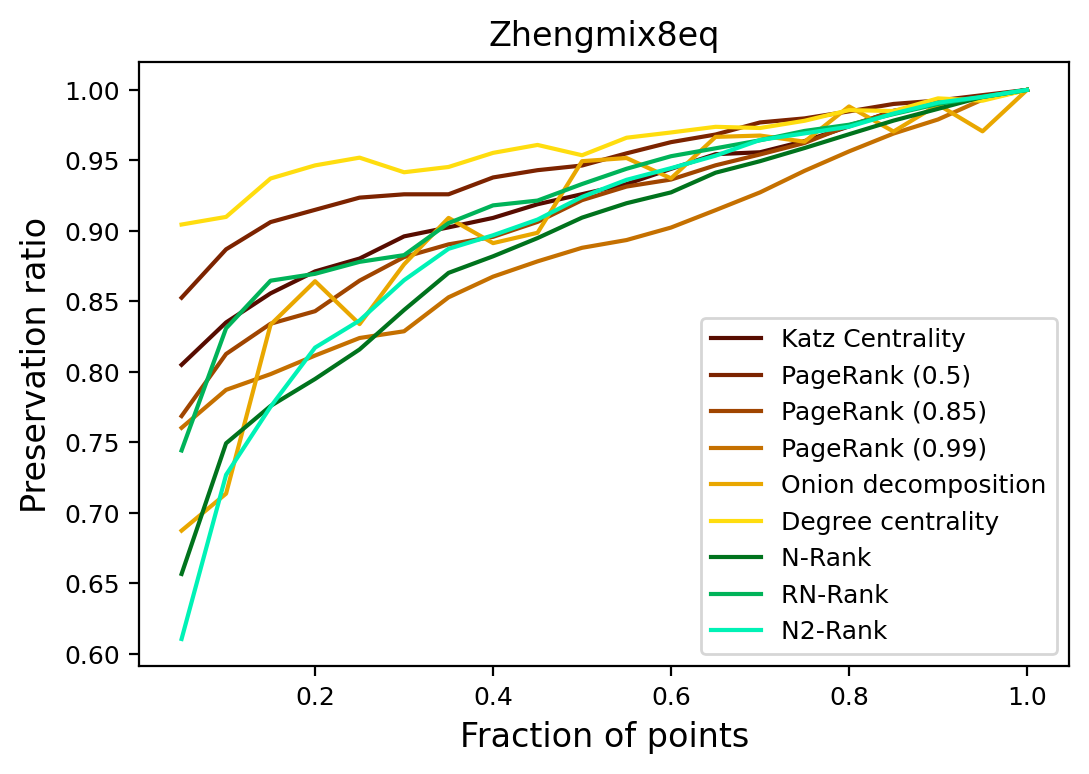}}
    \caption{Zhengmix8eq dataset}
    \label{fig:scRNA-6}
\end{figure}

\begin{figure}[H]
    \centering
\subfigure{\includegraphics[scale=0.45]{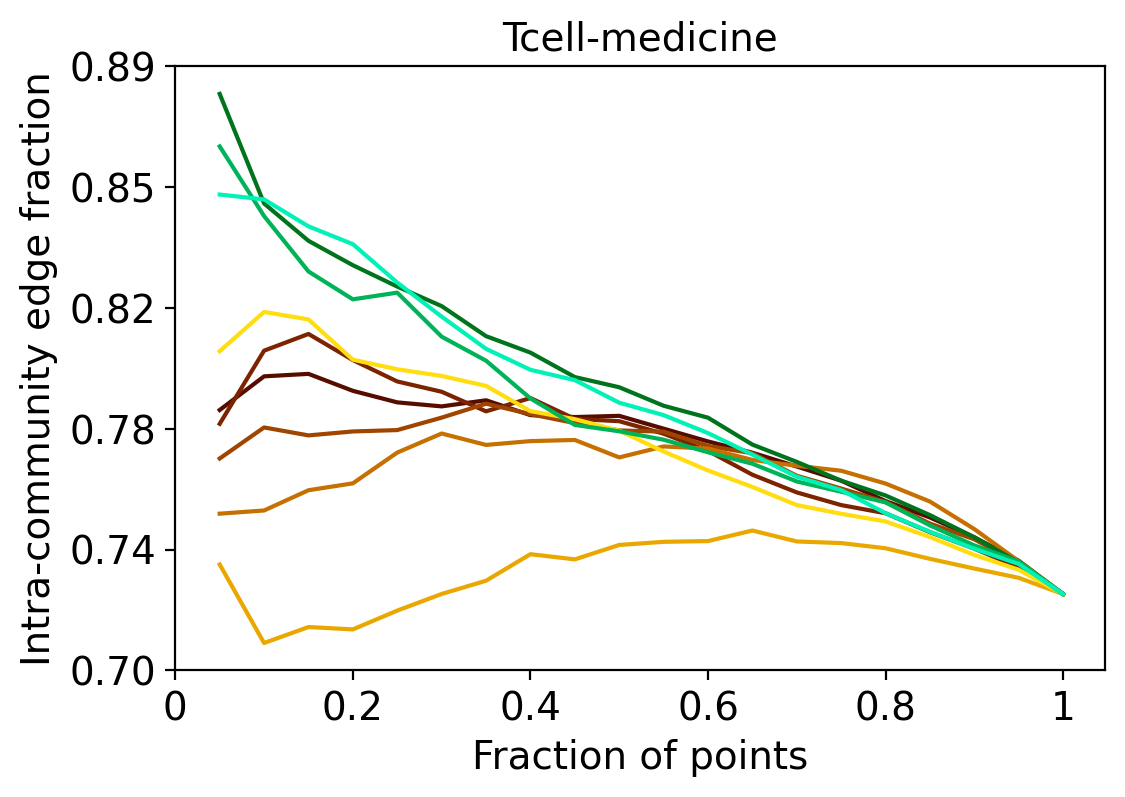}}
\subfigure{\includegraphics[scale=0.45]{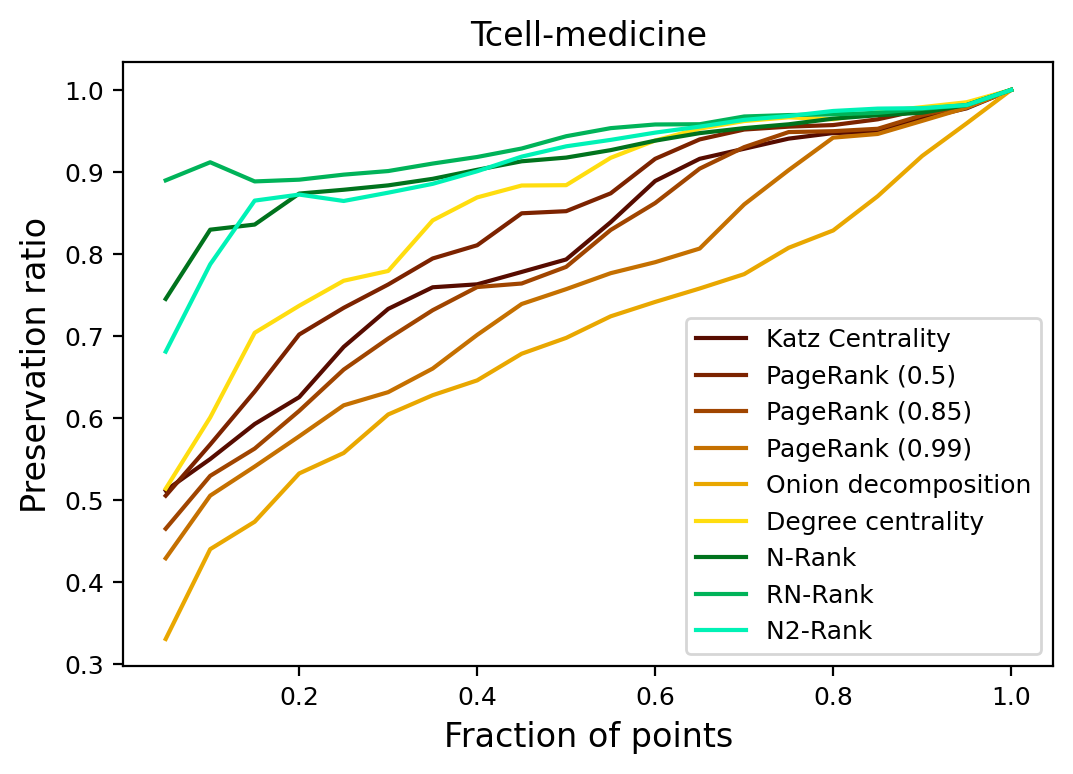}}
    \caption{Tcell dataset}
    \label{fig:scRNA-7}
\end{figure}

\begin{figure}[H]
    \centering
\subfigure{\includegraphics[scale=0.45]{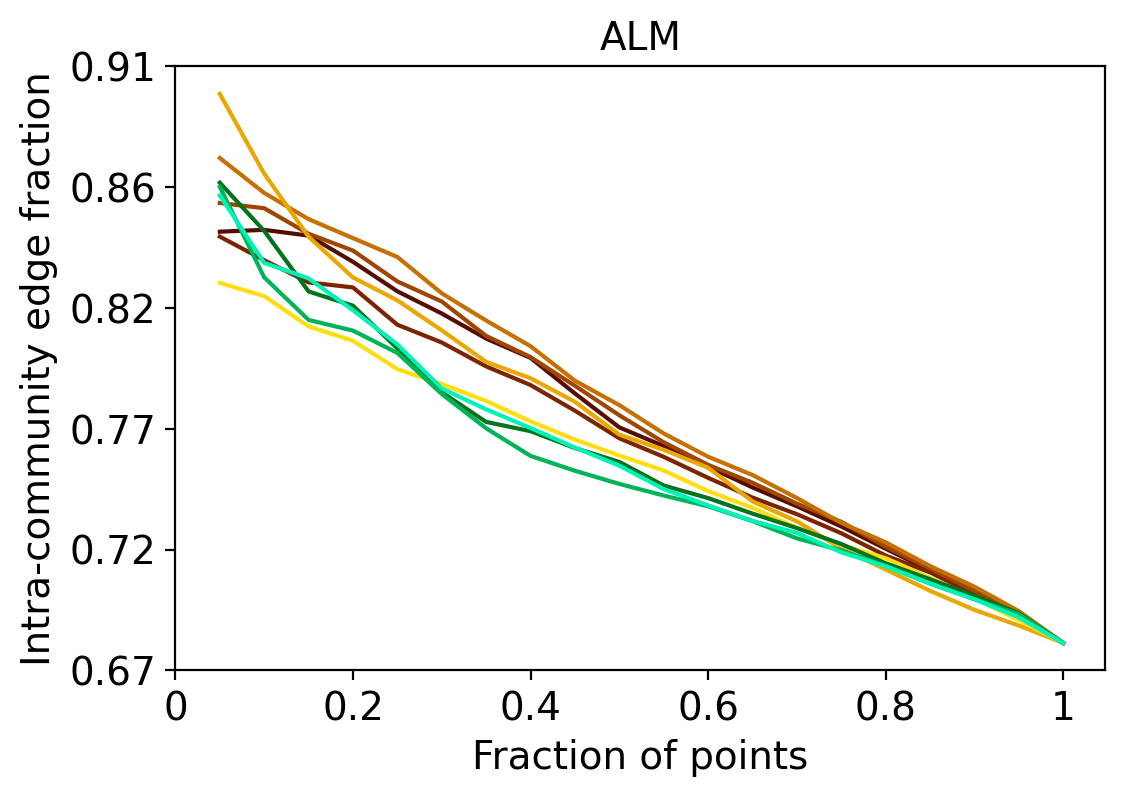}}
\subfigure{\includegraphics[scale=0.45]{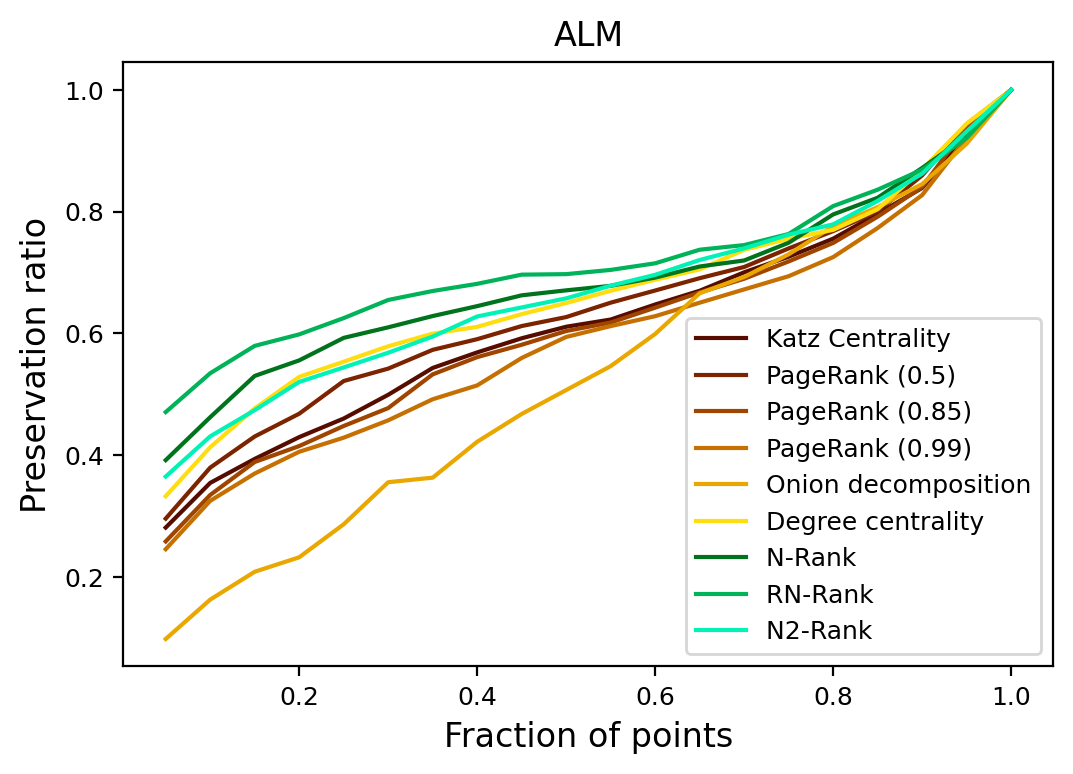}}
    \caption{ALM dataset}
    \label{fig:scRNA-8}
\end{figure}

\begin{figure}[H]
    \centering
\subfigure{\includegraphics[scale=0.45]{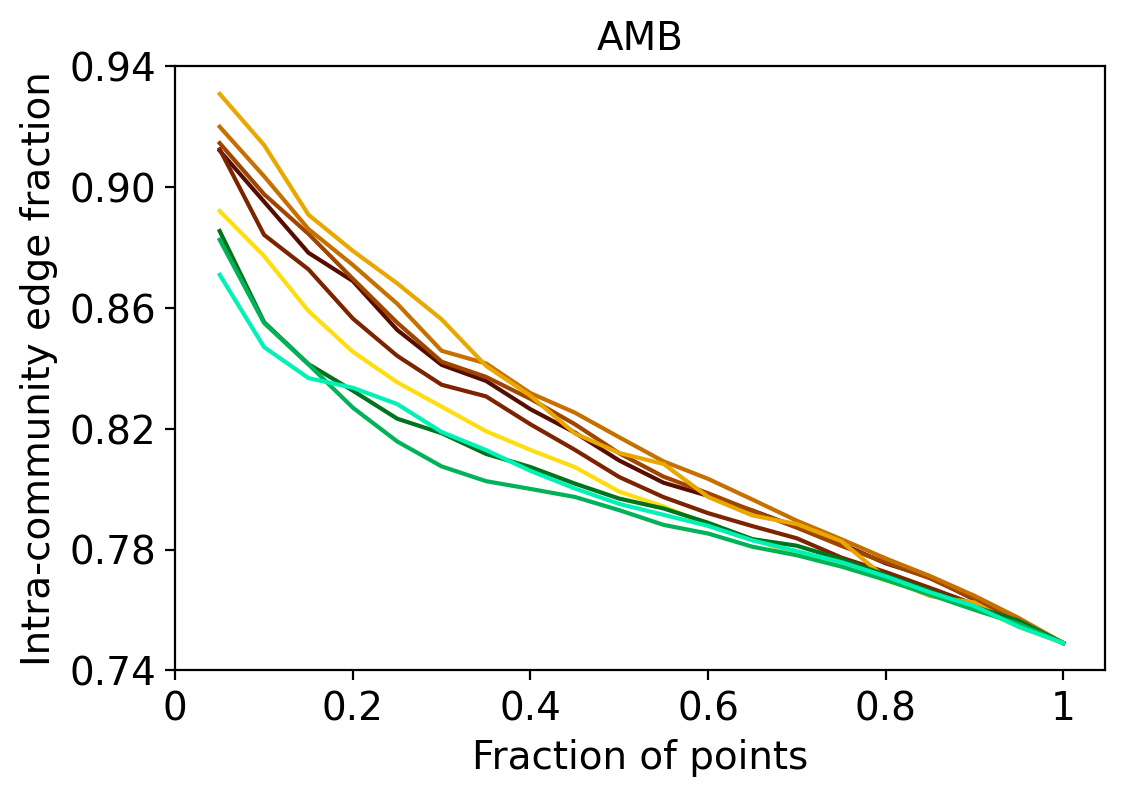}}
\subfigure{\includegraphics[scale=0.45]{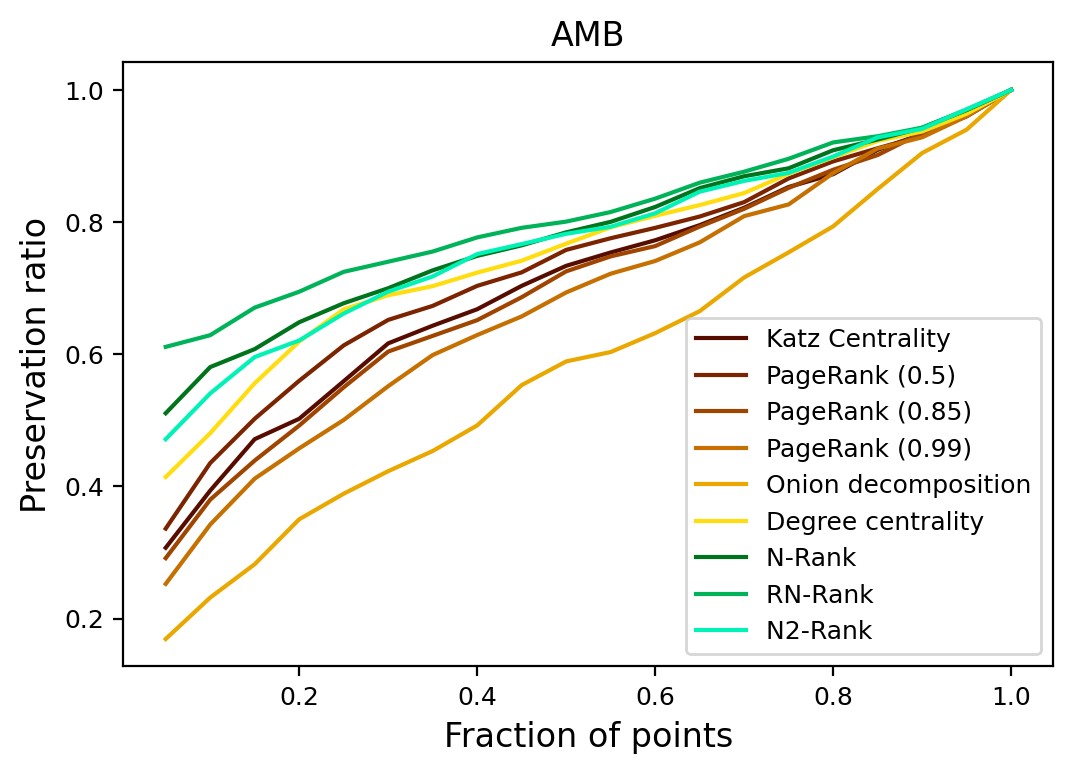}}
    \caption{AMB dataset}
    \label{fig:scRNA-9}
\end{figure}

\begin{figure}[H]
    \centering
\subfigure{\includegraphics[scale=0.45]{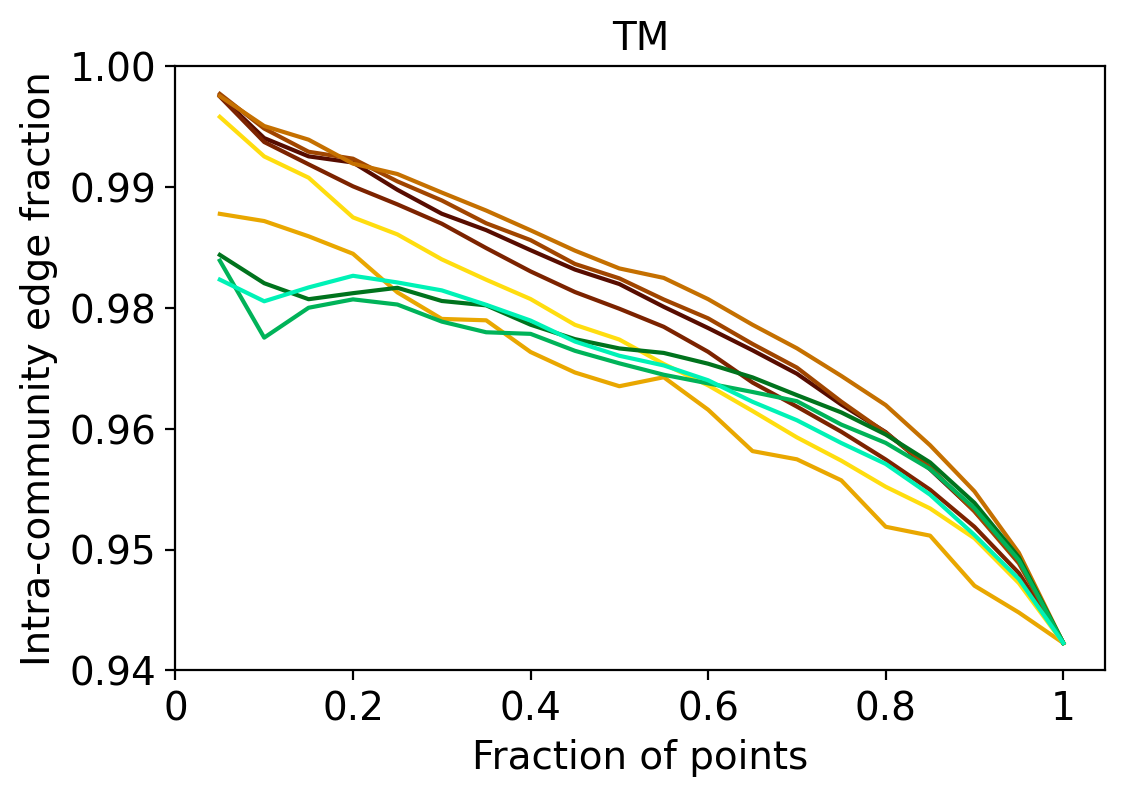}}
\subfigure{\includegraphics[scale=0.45]{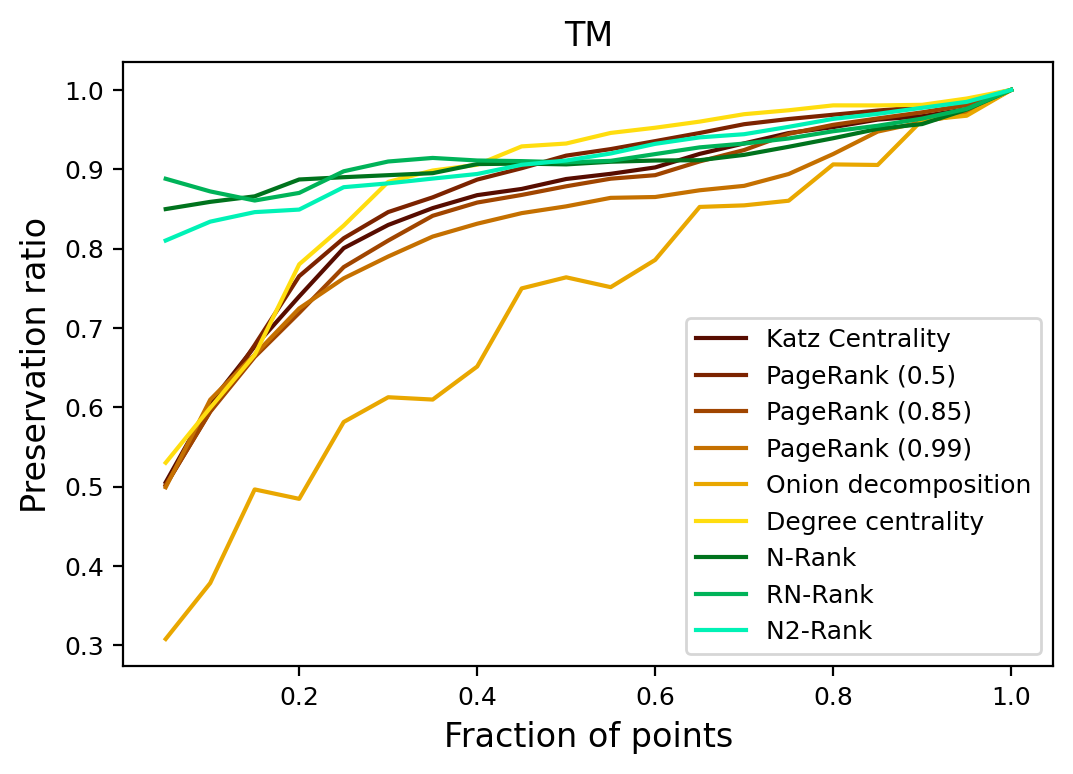}}
    \caption{TM dataset}
    \label{fig:scRNA-10}
\end{figure}

\begin{figure}[H]
    \centering
\subfigure{\includegraphics[scale=0.45]{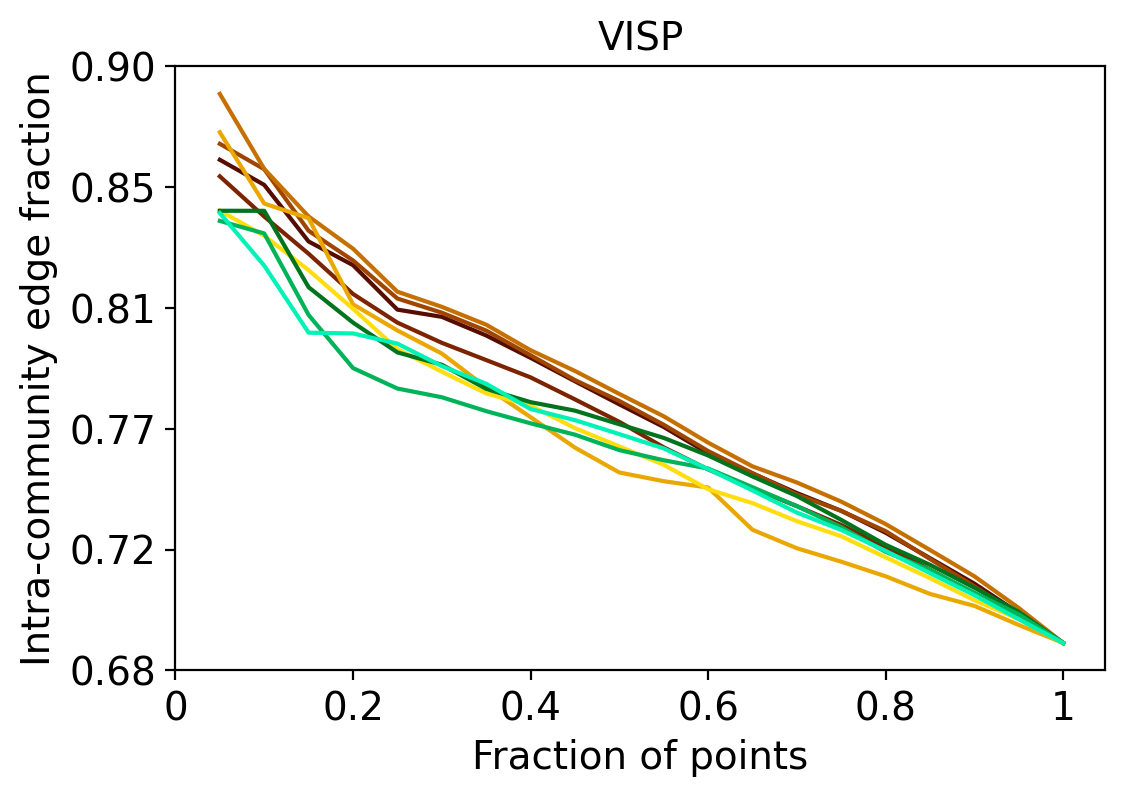}}
\subfigure{\includegraphics[scale=0.45]{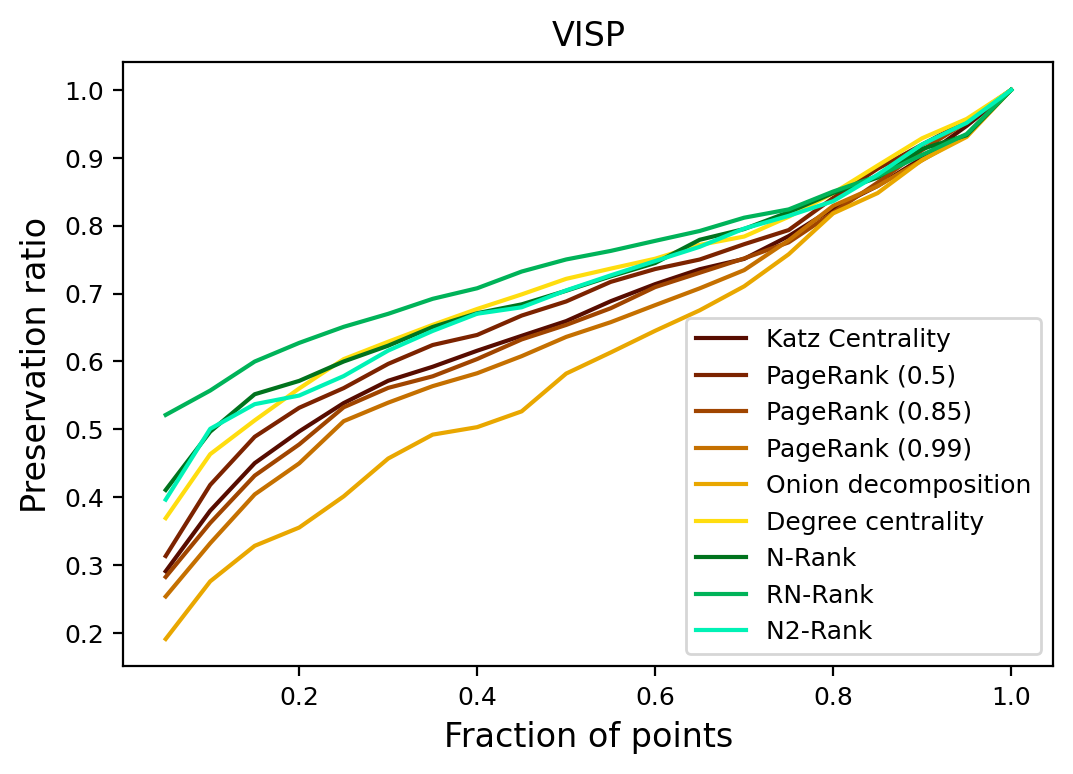}}
    \caption{VISP dataset}
    \label{fig:scRNA-11}
\end{figure}

\end{document}